%% file: main.tex
\documentclass[11pt,oneside]{article}

\input{mysty}

\usepackage {color}
\usepackage{lmodern}
\usepackage{amsmath,amssymb}
\usepackage{microtype}
\usepackage{geometry}
\usepackage{fontenc}

\begin{document}

\title{JustOPD: On-Policy Distillation from Autoregressive to Diffusion Language Models} 
\author{
Zikun Zhang\thanks{Department of Industrial Engineering and Operations Research, Columbia University, New York, NY 10027.
Emails: \texttt{zz3367@columbia.edu}, \texttt{js6646@columbia.edu},
\texttt{yao@columbia.edu}, \texttt{wt2319@columbia.edu}.}
\and
Jiayuan Sheng\footnotemark[1]
\and
David D. Yao\footnotemark[1]
\and
Wenpin Tang\footnotemark[1]
}
\date{}
\maketitle
\maketitle

\begin{abstract}
On-policy distillation (OPD) is a post-training approach for large language models (LLMs) that enables a student model to learn a teacher model's distribution from sequences sampled from the student's own policy. Despite its promising capability, OPD for diffusion language models (DLMs) remains underexplored and existing studies primarily focuses on DLM-to-DLM distillation. The limited availability of strong teacher DLMs and the flexible generation order of DLMs impose additional technical challenges as well. However, it remains unclear how to effectively distill knowledge from a strong autoregressive language model (ARM) into a weaker DLM, given the mismatch between the left-to-right decoding of ARMs and the flexible decoding order of DLMs.
Inspired by \cite{ni2026the}, we propose JustOPD, an ARM-to-DLM on-policy distillation method that uses autoregressive rollouts during training to align the student's generation process with the ARM teacher's left-to-right decoding. This alignment enables effective distillation from ARM teachers while preserving the DLM's flexible generation-order capability at inference time. Experiments using Qwen3-14B as the teacher and Dream-7B as the student demonstrate that JustOPD consistently improves mathematical reasoning performance over the base DLM across multiple divergence choices.

\end{abstract}

% {\crd
% \section*{To do list:}
% \input{todo}
% }

\section{Introduction}

On-policy distillation (OPD) \citep{agarwal2024policy,lu2025onpolicydistillation} has recently emerged as a key component of LLM post-training, where the student learns directly from the teacher's distribution on responses sampled from the student's own policy, while the teacher provides dense supervision signals for these student-generated responses. For autoregressive language models (ARMs), OPD is relatively straightforward, as the student and teacher can be chosen from the same model family (e.g., the Qwen3 series), leveraging the many strong and readily available open-source ARMs. Moreover, the fixed left-to-right, token-by-token generation process of ARMs naturally aligns the generation trajectories of the student and teacher, making ARM-to-ARM OPD fundamentally compatible.

In contrast, applying OPD to mask-based diffusion language models (DLMs) introduces additional challenges. In particular, the arbitrary-order generation process of DLMs requires careful consideration of the student's generation trajectory, while strong and readily available DLM teachers remain scarce. \citet{ren2026trace} first studied DLM-to-DLM distillation by querying a stronger teacher DLM with the partially denoised sequences generated by the student, which naturally accommodates the arbitrary generation order of DLMs. However, explicitly recording and replaying the generation trajectory of each student sample can be computationally expensive, particularly for long sequences. Furthermore, restricting the teacher to be a DLM limits the scalability and generality of this approach, given the scarcity of strong, open-source DLMs. Overall, the flexibility of arbitrary-order generation introduces substantial complexity to DLM post-training, particularly in aligning the student's generation trajectory with that of the teacher.

Given the limited availability of strong DLM teachers, an appealing alternative is to distill knowledge from a well-trained ARM into a weaker DLM. This direction is particularly promising because ARMs are substantially more mature and widely adopted in the LLM community, with a large number of high-performing models available for distillation. Moreover, ARM-to-DLM distillation offers an attractive way to combine the complementary strengths of the two paradigms: ARMs generally achieve stronger generation quality, while DLMs enable parallel token generation and can substantially reduce generation latency compared with ARMs under comparable hardware and engineering optimization.
However, transferring knowledge from an ARM to a DLM is fundamentally challenging due to their incompatible generation paradigms. An ARM is trained with a causal mask and generates tokens strictly from left to right, conditioning each token on a complete prefix, while a DLM progressively denoises a sequence and can predict tokens in an arbitrary order. As a result, the partially denoised sequences at intermediate denoising steps encountered by a DLM do not generally correspond to valid prefixes under the ARM's training distribution, making it unclear how an ARM teacher can provide meaningful supervision throughout the DLM's denoising process. This mismatch motivates the development of effective ARM-to-DLM distillation methods.

To bridge this gap, we propose to enforce left-to-right rollout generation for the DLM during training, so that the intermediate student states remain valid prefixes that can be directly processed by an ARM teacher. This idea is inspired by JustGRPO \citep{ni2026the}, which observes that arbitrary-order generation during reinforcement learning (RL) rollouts may hinder the reasoning capability of DLMs: the model may generate easier future tokens with high confidence before resolving the more difficult tokens that depend on them. JustGRPO therefore adopts autoregressive-style rollouts during RL training and demonstrates substantial performance gains. Motivated by these observations, we similarly leverage autoregressive rollouts during OPD, which we term \textbf{JustOPD}, to bridge the generation-paradigm gap between ARMs and DLMs, while preserving the flexible arbitrary-order generation capability of DLMs at inference time. We conduct experiments by choosing Dream-7B \citep{ye2025dream} as the base student DLM and Qwen3-14B \citep{yang2025qwen3} as the teacher ARM, because Dream-7B uses a full-attention diffusion model but initialized from a pretrained ARM checkpoint (Qwen2.5-7B weights). Our experiment results on several mathematical reasoning benchmarks (GSM8K and MATH500) demonstrate the effectiveness of JustOPD.

\subsection{Contributions}
Our contributions are summarized as follows:
\begin{enumerate}[leftmargin=*]
\item
We propose JustOPD, an effective ARM-to-DLM on-policy distillation method that transfers knowledge from an autoregressive teacher to a diffusion student. JustOPD enforces autoregressive-style rollouts for the DLM student during training, aligning its intermediate generation trajectory with the ARM teacher and thereby enabling the teacher to provide meaningful token-level supervision.
We instantiate JustOPD with a sampled-token, PPO-style objective that supports general $f$-divergences, and we evaluate the loss only at the $25\%$ of positions where the teacher’s entropy is highest, removing $75\%$ of the student forward passes.

\item 
We empirically validate JustOPD using Dream-7B as the student and Qwen3-14B as the teacher on GSM8K and MATH500. After noting that strong ARM teachers suffer a sharp drop in performance under the short generation lengths typical of DLMs, we fine-tune Qwen3-14B to generate concise responses while preserving accuracy, making it an effective teacher. JustOPD consistently improves student performance across all tested settings, including different $\alpha$-divergences and the Jensen–Shannon divergence. We evaluate these gains using the DLM's native any-order decoder, demonstrating that knowledge distilled through autoregressive rollouts transfers effectively to arbitrary-order inference.

% We demonstrate the effectiveness of ARM-to-DLM OPD through extensive experiments. Using Dream-7B as the student DLM and Qwen3-14B as the teacher ARM, we evaluate JustOPD on multiple mathematical reasoning benchmarks, including GSM8K and MATH500, and show consistent improvements over the base DLM.

\end{enumerate}

\subsection{Related Works}
\textbf{On-Policy Distillation.}
OPD was originally proposed to fine-tune ARMs by pushing the token-level distribution of the student model to that of the teacher model. \citet{agarwal2024policy} introduced Generalized Knowledge Distillation, which formulates on-policy knowledge distillation as a supervised learning problem by training the student on sequences sampled from its own policy while using the teacher to provide token-level supervision. \citet{lu2025onpolicydistillation} further studied an RL-style formulation of OPD based on KL divergence, establishing a connection between on-policy distillation and policy optimization and showing that OPD can serve as a general framework for improving a student model using dense supervision from a stronger teacher. \citet{jin2026entropyaware} showed that the mode-seeking reverse KL collapses generation diversity and yields unstable learning signals at tokens where the teacher distribution has high entropy, so they propose Entropy-Aware On-Policy Distillation, which adds a forward KL term and computed over the teacher's top-k tokens, only at such high-entropy positions. \citet{zhao2026self} proposed On-Policy Self-Distillation (OPSD), in which a single ARM acts as both teacher and student: the teacher is conditioned on privileged information, the student sees only the question, and training minimizes the per-token divergence between the two along the student's own rollouts. Recent studies investigate the mechanisms and data requirements of OPD \citep{li2026rethinking,ding2026does,yu2026understanding}. Complementary work shows that even a single training query can recover most of the gains of full-data OPD \citep{fu2026rethinking}. \citet{song2026survey} reviews recent progress of OPD methods for ARMs. 

\textbf{OPD for DLMs.} OPD for DLMs is limited in the existing literature. TOPD \citep{ren2026trace} distills a stronger DLM teacher on the committed decisions along the DLM student's own denoising trace. HOPD \citep{huang2026hesitation} extends TOPD by matching teacher distributions at every masked position in each denoising step, while the decoder commits only to a confident subset of the DLM's proposed tokens.
Recent works also extend OPSD to DLMs by sourcing the teacher's privileged information from the model's own generation. d-OPSD \citep{luo2026learning} conditions the self-teacher on partially revealed tokens of the student's final answer as a suffix and supervises at the denoising-step level. dOPSD \citep{dat2026dopsd} uses later, more-decoded states of the student's own denoising trajectory as a peek-ahead teacher for an earlier intermediate state. Both report gains over SFT and RL-based post-training of DLMs. In addition, OPDLM \citep{su2026data} instead uses OPD to convert an ARM into a DLM under self-distillation, supervising the student's own denoising states with the frozen ARM teacher queried on causal prefixes of the student's final output. ForkLeft \citep{liu2026forkleft} also resolves the ARM-to-DLM mismatch by comparing the causal teacher and the bidirectional student only on student-generated prefixes, but it generates training rollouts in an entropy-first order and applies supervision at uniformly sampled positions. JustOPD instead uses strictly left-to-right rollouts, so the student and teacher share the same autoregressive factorization and the standard OPD objective applies directly, and it uses teacher entropy not to order the rollout but to select which positions receive supervision.

\textbf{Post-Training of DLMs.}
Due to fundamental modeling differences between the auto-regressive and diffusion generative paradigms for large language models, DLMs enable arbitrary-order decoding at the cost of intractable or computationally prohibitive trajectory likelihood, driving a line of work \citep{zhao2025diffpo,zhao2025d,yang2026mmada,wang2026d2,wang2026revolutionizing,ou2026principled,ni2026the,zhang2026continuous} focusing on improving the reasoning capabilities of DLMs using RL fine-tuning in recent years. In particular,
\citet{ni2026the} indicated that arbitrary-order decoding can narrow rather than expand the reasoning potential of DLMs: confidence-based samplers defer high-entropy ``forking'' tokens until future context has resolved their uncertainty, which lowers Pass@$k$ relative to left-to-right decoding. They therefore propose JustGRPO, which treats the DLM as an autoregressive policy during RL training so that standard GRPO applies with exact likelihoods while keeping parallel decoding at inference. OPD is closely tied to RL: with the teacher–student log-probability gap as a dense per-token reward, it reduces to a policy-gradient objective that can be optimized in standard RL pipelines \citep{gu2024minillm,lu2025onpolicydistillation,zhao2026self,jin2026entropyaware}, and G-OPD \citep{yang2026learning} further generalizes it into a KL-regularized RL objective whose reward can be extrapolated to surpass the teacher. Our work is closely related to the autoregressive rollout strategy introduced by \citet{ni2026the} and the entropy-aware generation strategy proposed by \citet{jin2026entropyaware}, both of which address the challenges arising from arbitrary-order generation in diffusion language models during post-training.

\section{Preliminaries}

\textbf{Notations.}
We use lowercase letters to denote scalars, and boldface lowercase letters to denote vectors. The $i$-th entry of a vector $\x$ is denoted by $\x^i$, or simply $x^i$ when the context is clear, and $\x^{<i}$ denotes $(x^1,\cdots,x^{i-1})$. $\e_j$ refers to a one-hot column vector with a $1$ in the $j$-th position. We use $\mathrm{Cat}(\cdot;\p)$ for a categorical distribution over the vocabulary with probabilities given by $\p$.

\textbf{Diffusion Language Models.} We consider masked diffusion models \citep{shi2024simplified,sahoo2024simple,lou2023discrete,campbell2022continuous}.
Let $\cS=(\cV\cup\{\mathtt{m}\})^L$ be the state space, where $\cV$ is a discrete vocabulary space, $\mathtt{m}$ is the special ``mask'' token, and $L$ is the sequence length.

The forward process of a DLM gradually corrupts the clean data distribution $p_\mathrm{data}$ over $\cS$ into pure mask noise according to a continuous-time transition kernel. At noise level $\tau\in[0,1]$, $\x_\tau$ is a partially-noised sequence, which is obtained by independently masking each dimension of $\x_0\sim p_\mathrm{data}$ with probability $1-\alpha_\tau$. Here,
$\alpha_\tau$ is the noise scheduled to satisfy $\alpha_0\approx 1$ and $\alpha_1\approx 0$ and a common choice is $\alpha_\tau=1-\tau$ \citep{ye2025dream,nie2026large}. 

The reverse process gradually recovers clean tokens starting from the masks. Specifically, a denoiser $\p_\theta(\cdot)$ is trained using a cross-entropy loss \citep{shi2024simplified,sahoo2024simple,ye2025dream,nie2026large}
to predict the distributions of all masked tokens simultaneously. To be more specific, given the current partially denoised sequence $\x_\tau$, for any $i\in [L]$ with $x_\tau^i=\mathtt{m}$, $\p^{(i)}_\theta(\x_\tau)$, the $i$-th output of $\p_\theta(\x_\tau)$, predicts the conditional distribution $p_{0|\tau}(x_0^i=\cdot\mid \x_\tau)$. At inference time, the continuous-time reverse process is discretized over a set of time steps indexed by $t\in \Z_{\geq 0}$, and various strategies can be used to decide which masked tokens to unmask at each denoising step, such as selecting the $k$ tokens with the highest confidence \citep{chang2022maskgit}, the largest margin between the top two probabilities \citep{kim2025train}, or the lowest predictive entropy \citep{ye2025dream,ben2026accelerated}. More details are deferred to Appendix~\ref{app:denoising}.

\textbf{On-Policy Distillation for ARMs.} In ARMs, given a prompt $\q$, the model is parametrized to predict the next-token distribution $p_\theta(\cdot\mid \q,\x^{<t})$ conditioned on the prefix $(\q,\x^{<t})$, so the probability of a response $\x$ factorizes as $p_\theta(\x\mid \q)=\prod_{t=1}^{|\x|}p_\theta(x^t\mid \q,\x^{<t})$. Let $f\in C^1(0,+\infty)$ be a convex function with $f(1)=0$, and define the $f$-divergence $D_f(\x\,\Vert\,\y):=\sum_i y^i\,f\big(x^i/y^i\big)$. Given a teacher ARM $p_{\mathtt{te}}$ and a student ARM $p_\theta$, OPD minimizes the token-level $f$-divergence between the teacher and the student along the student's own trajectories:
\begin{equation}\label{f_div_objective}
    \min_\theta \E_{\q\sim \cD,\ \x\sim p_\theta(\cdot\mid \q)}\frac{1}{|\x|}  \sum_{t=1}^{|\x|}  D_f\big(p_{\mathtt{te}}(\cdot\mid \q,\x^{<t})\,\Vert\, p_\theta(\cdot\mid \q,\x^{<t})\big).
\end{equation}
By evaluating the teacher and student distributions on prefixes generated by the student itself, OPD directly trains the student on the states it is likely to encounter during inference, enabling it to iteratively improve through the teacher's guidance on its own outputs. 
Typical choices of $f$ include $f(u)=-\log u$ (reverse KL), $f(u)=u\log u$ (forward KL), $f(u)=\frac{u^\alpha-\alpha u - (1-\alpha)}{\alpha (\alpha-1)}$ with $\alpha \in (0,1)$ ($\alpha$-divergence, which recovers centralized reverse KL as $\alpha\to0^+$ and centralized forward KL as $\alpha\to1^-$, and $f(u)=\frac{1}{2}\big(u\log u-(1+u)\log \frac{1+u}{2}\big)$ (Jensen-Shannon divergence).

\section{JustOPD}
\textbf{Motivation.} Our goal is to distill the knowledge from a stronger ARM to a weaker DLM. \citet{ni2026the} argue that arbitrary-order generation during RL rollouts may limit the reasoning capability of DLMs, as it may generate easier future tokens with high confidence before determining the more difficult logical dependencies leading to them, which are inherently high-entropy. Hence, during training, we consider a token-level autoregressive policy for the student DLM, where a sequence $\x$ is generated via left-to-right decoding. Specifically, given a prompt $\q$, for $t=1,2,\cdots$, the student policy follows
\begin{equation}\label{ARpolicy}
\pi_\theta(\cdot\mid \q,\x^{<t})
:=\mathrm{Cat}\left(\cdot;\p_\theta^{(t)}(\q,\x^{<t},\mathtt{m},\cdots,\mathtt{m})\right),
\end{equation}
where the denoiser $\p_\theta^{(t)}(\cdot)$ takes as input the prefix $(\q,\x^{<t})$ concatenated with a masked future suffix and predicts the clean token at the $t$-th position.
The resulting sequence-level policy has the token-level factorization:
$$
\pi_\theta(\x\mid \q):=\prod_{t=1}^{|\x|} \pi_\theta(x^t\mid \q,\x^{<t}).
$$
Let $\pi_{\mathtt{te}}$ be an ARM teacher model with the natural trajectory likelihood factorization:
$$
\pi_{\mathtt{te}}(\x\mid \q):=\prod_{t=1}^{|\x|} \pi_{\mathtt{te}}(x^t\mid \q,\x^{<t}).
$$
As a result, the generation order of the ARM teacher and DLM student coincides, which enables us to use the OPD objective~\eqref{f_div_objective} developed for ARMs by simply replacing $p_\theta$ by $\pi_\theta$ and $p_{\mathtt{te}}$ by $\pi_{\mathtt{te}}$. 

\textbf{Full/TopK-Vocabulary Distillation.}
In practice, for training stability, the $f$-divergence in \eqref{f_div_objective} is clipped \citep{zhao2026self} with a constant $C>0$:
$$
D_f^{\text{clip}}\big(\pi_{\mathtt{te}}(\cdot\mid \q,\x^{<t})\,\Vert\, \pi_\theta(\cdot\mid \q,\x^{<t})\big):=\min \left(D_f\big(\pi_{\mathtt{te}}(\cdot\mid \q,\x^{<t})\,\Vert\, \pi_\theta(\cdot\mid \q,\x^{<t})\big), C\right),
$$
which prevents excessively large divergence values from inducing high-variance or unstable gradient updates. Alternatively, the divergence can be computed only over the teacher's top-$K$ tokens \citep{jin2026entropyaware}:
$$
D_f^{\text{TopK}}\big(\pi_{\mathtt{te}}(\cdot\mid \q,\x^{<t})\,\Vert\, \pi_\theta(\cdot\mid \q,\x^{<t})\big):=\sum_{v\in \mathrm{Top}K(\pi_{\mathtt{te}}(\cdot\mid \q,\x^{<t}))} \pi_{\theta}(v\mid \q,\x^{<t})\, f\!\left(\frac{\pi_{\mathtt{te}}(v\mid \q,\x^{<t})}{\pi_{\theta}(v\mid \q,\x^{<t})}\right).
$$
The on-policy training objective is
\begin{equation}\label{topk_div}
\min_\theta \E_{\q\sim \cD,\ \x\sim \pi_{\theta_\mathrm{old}}(\cdot\mid \q)}\frac{1}{|\x|}\sum_{t=1}^{|\x|} D_f^{\dagger}\big(\pi_{\mathtt{te}}(\cdot\mid \q,\x^{<t})\,\Vert\, \pi_\theta(\cdot\mid \q,\x^{<t})\big),
\end{equation}
where $\dagger\in \{\text{clip},\text{TopK}\}$. The pseudocode for full/top-$K$ vocabulary distillation is presented in Algorithm~\ref{alg:full_v} in Appendix~\ref{app:alg}. However, computing the $f$-divergence at the vocabulary level remains computationally expensive, as it requires evaluating the teacher and student probabilities over a large number of vocabulary tokens at every sequence position, which is even more computationally prohibitive for DLMs. We therefore focus on a more computationally efficient sampled-token distillation approach below in this work.

\textbf{Sampled-Token Distillation.}
In objective~\eqref{f_div_objective}, let $u_{\theta,t}:=\pi_{\mathtt{te}}(x^t\mid \q,\x^{<t})/\pi_{\theta}(x^t\mid \q,\x^{<t})$ for a sampled student-generated sequence $\x$. Since $x^t\sim \pi_\theta(\cdot\mid \q,\x^{<t})$, we have $\E_{x^t}[f(u_{\theta,t})]=D_f\big(\pi_{\mathtt{te}}(\cdot\mid \q,\x^{<t})\,\Vert\, \pi_\theta(\cdot\mid \q,\x^{<t})\big)$, so the $f$-divergence admits a single-sample Monte Carlo (MC) estimate \citep{gu2024minillm,lu2025onpolicydistillation,zhao2026self,zhao2026opd+,jin2026entropyaware}:
\begin{equation}\label{OPDobjective}
    \min_\theta \E_{\q\sim \cD,\ \x\sim \pi_\theta(\cdot\mid \q)}\frac{1}{|\x|}  \sum_{t=1}^{|\x|}  f\left(u_{\theta,t}\right).
\end{equation}

The objective \eqref{OPDobjective} can be optimized with policy-gradient methods in RL. In particular, the Proximal Policy Optimization (PPO)-type \citep{schulman2017proximal} surrogate objective is
\begin{equation}\label{PPOloss_ARM}
\max_\theta \E_{\q\sim \cD,\ \x\sim \pi_{\theta_\mathrm{old}}(\cdot\mid \q)} \frac{1}{|\x|}\sum_{t=1}^{|\x|} \min\left(
\rho_t^\theta \cdot A_t,\ \mathrm{clip}(\rho_t^\theta,1-\epsilon,1+\epsilon)\cdot A_t
\right),
\end{equation}
where $\rho_t^\theta:=\pi_\theta(x^t\mid \q,\x^{<t})/\pi_{\theta_\mathrm{old}}(x^t\mid \q,\x^{<t})$ corrects for sampling under $\pi_{\theta_\mathrm{old}}$, and
$
A_t:=-f(u_{\theta_\mathrm{old},t})+u_{\theta_\mathrm{old},t}\, f'(u_{\theta_\mathrm{old},t})
$
is the per-token corrected advantage \citep{zhao2026opd+} of the sampled token $x^t$ (also see Proposition~\ref{prop:corrected_adv}).

\textbf{Centralized $f$-Divergence.}
One may be concerned that the single-sample Monte Carlo estimation in \eqref{OPDobjective} introduces errors compared with the full-vocabulary calculation. Here, we argue that for any $f$-divergence, one can choose its unique centralized version, which improves the signal-to-noise ratio of the estimation while retaining the corresponding advantage up to a constant, so the algorithm remains the same.
For any convex $f\in C^1(0,+\infty)$ with $f(1)=0$, let
\[
\tilde{f}(u):=f(u)-f(1)-f'(1)(u-1),\quad \forall u\in (0,+\infty).
\]
$\tilde f$ is still convex and it holds that $\tilde{f}(1)=\tilde{f}'(1)=0$, so $\tilde f$ is a \textit{centralized} version of $f$. The following Proposition~\ref{prop:signal_to_noise} shows that centralized $f$-divergence improves the local signal-to-noise ratio of Monte Carlo estimation, with the proof deferred to Appendix~\ref{signal_to_noise}.

\begin{table}[!ht]
\centering
\renewcommand{\arraystretch}{1.0}
\begin{tabular}{lccc}
\toprule
\textbf{Divergence}
& \textbf{$f(u)$}
& \textbf{ $\tilde f(u)$} & \textbf{ $A_f(u)$} \\
\midrule
Forward KL
& $u\log u$
& $u\log u-u+1$
& $u$ \\
Reverse KL
& $-\log u$
& $-\log u+u-1$
& $\log u - 1$ \\
$\alpha$-divergence, $\alpha\in(0,1)$
& $\dfrac{u^\alpha-\alpha u-(1-\alpha)}{\alpha(\alpha-1)}$
& unchanged
& $\dfrac{u^\alpha-1}{\alpha}$ \\
Jensen-Shannon
& $\frac{1}{2}\big(u\log u-(1+u)\log\frac{1+u}{2}\big)$
& unchanged
& $\frac{1}{2}\log\frac{1+u}{2}$ \\
\bottomrule
\end{tabular}
\caption{Typical $f$-divergences and their centralized versions $\tilde f$ and per-token advantages $A_f$.}
\label{tab:f_div_advantage}
\end{table}

\begin{proposition}\label{prop:signal_to_noise}
Fix a prefix $(\q,\x^{<t})$. Let $P_t(\cdot)=\pi_{\mathtt{te}}(\cdot\mid \q,\x^{<t})$ and
$Q_{\theta,t}(\cdot)=\pi_\theta(\cdot\mid \q,\x^{<t})$ be normalized, strictly positive distributions
on the same finite support. Define
\[
 u_{\theta,t}(x):=\frac{P_t(x)}{Q_{\theta,t}(x)}
 \quad\text{and}\quad D_{f,t}(\theta):=\mathbb{E}_{x'\sim Q_{\theta,t}}[f(u_{\theta,t}(x'))].
\]
Assume that there exists $\gamma \in (0,1)$ such that $|u_\theta(x)-1|\leq \gamma$ for all $x$ in the support of $Q_{\theta,t}$, and that $f$ is four times continuously differentiable on a neighborhood of $[1-\gamma,1+\gamma]$ with $f''(1)>0$. Let $\sigma_{f,t}^2(\theta):=\mathrm{Var}_{x'\sim Q_{\theta,t}}[f(u_{\theta,t}(x'))]$ be the variance of the single-sample estimate. Then we have the local noise-to-signal ratios $\sigma_{f,t}(\theta)/D_{f,t}(\theta)=\Theta(1/\gamma)$ and $\sigma_{\tilde f,t}(\theta)/D_{\tilde f,t}(\theta)=O(1)$. 
\end{proposition}

The following proposition characterizes the effect of centralizing the $f$-divergence on the corresponding policy gradient and corrected advantage function; the proof is provided in Appendix~\ref{app:prop2}.

\begin{proposition}\label{prop:corrected_adv}
Let the advantage function \citep{zhao2026opd+} be
$
A_f(u):=-f(u)+u f'(u)
$ for all $u\in(0,+\infty)$.
We have:
\begin{enumerate}
\item[(i)] The gradient of the negative divergence satisfies
\[
\nabla_\theta\big(-D_{f,t}(\theta)\big)=\mathbb{E}_{x\sim Q_{\theta,t}}\left[A_f(u_{\theta,t}(x))\,\nabla_\theta\log Q_{\theta,t}(x)\right];
\]
\item[(ii)] $f$ and $\tilde f$ induce the same divergence, and their corrected advantages differ by a constant:
\[
D_{\tilde f,t}(\theta)=D_{f,t}(\theta)\quad \text{and}\quad A_{\tilde f}(u)=A_f(u)-f'(1),\quad \forall u\in(0,+\infty).
\]
\end{enumerate}
\end{proposition}

Typical centralized $f$-divergences along with the corrected advantages are summarized in Table~\ref{tab:f_div_advantage}. By Proposition~\ref{prop:corrected_adv} (ii), $f$ and its centralized version $\tilde f$ share the same OPD objective and equivalent corrected advantages up to a constant baseline $f'(1)$. As a result, OPD implicitly employs a centralized $f$-divergence, which improves the signal-to-noise ratio of the single-sample Monte Carlo estimator in \eqref{OPDobjective}.

\begin{algorithm}[!t]
\caption{JustOPD for DLMs: Sampled-Token Distillation}\label{alg}
\begin{algorithmic}[1]
\Require ARM teacher policy $\pi_{\mathtt{te}}$, DLM student policy $\pi_\theta$, PPO clip $\epsilon$, learning rate $\eta$, prompt dataset $\cD$
\For{each training iteration}
    \State Set $\pi_{\theta_{\text{old}}} \leftarrow \pi_\theta$
    \State Sample a prompt batch $\cB \subset \cD$
    \State Rollout buffer $\mathcal{R} \leftarrow \emptyset$
    \For{each prompt $\q \in \cB$}
        \State Sample $\x = (x^1, \cdots, x^{|\x|}) \sim \pi_{\theta_{\text{old}}}(\cdot \mid \q)$
        \For{$t = 1$ to $|\x|$}
            \State Query teacher: $\left(\pi_{\mathtt{te}}(x^t \mid \q,\x^{<t}), H_t^{\mathtt{te}}\right)$
        \EndFor
        \State $\cH^\mathtt{te}\leftarrow \{t\in[|\x|]: H_t^\mathtt{te}\in \mathrm{Top}_{25\%}(\{H_s^\mathtt{te}\}_{s=1}^{|\x|})\}$
        \State Store trajectory-level data:
        $
        \left(\q, \x, \cH^\mathtt{te}, \{\pi_{\mathtt{te}}(x^t \mid \q,\x^{<t})\}_{t\in \cH^\mathtt{te}}\right) \rightarrow \mathcal{R}
        $
    \EndFor
    \For{each mini-batch gradient step}
        \State Sample a mini-batch of prompts $\cB_{\text{mini}} \subset \mathcal{R}$
        \State Compute the loss using \eqref{finalobjective}:
        \[
        \mathcal{L}(\theta) = \frac{1}{ |\cB_{\text{mini}}|} \sum_{(\q,\x) \in \cB_{\text{mini}}} \frac{1}{|\cH^\mathtt{te}|}  \sum_{t\in \cH^\mathtt{te}} \min\left(
\rho_t^\theta \cdot A_t,\ \mathrm{clip}(\rho_t^\theta,1-\epsilon,1+\epsilon)\cdot A_t
\right)
        \]
        \State Update parameters: $\theta \leftarrow \theta + \eta \nabla_\theta \mathcal{L}(\theta)$
    \EndFor
\EndFor
\end{algorithmic}
\end{algorithm}

\textbf{Practical Considerations.} Unlike ARMs, DLMs must evaluate each position independently rather than in a single causal forward pass, so computing exact likelihoods $\{\rho_t^\theta\}_{t=1}^{|\x|}$ in \eqref{PPOloss_ARM} incurs extra per-iteration overhead. \cite{ni2026the} suggest that reasoning is steered by a sparse set of high-entropy forking tokens, and \cite{jin2026entropyaware} indicate that high-entropy tokens
in the teacher model are not merely noise but encode important knowledge in the form of multiple plausible reasoning paths and meaningful uncertainty. Thus, we evaluate $\rho_t^\theta$ only at the top-$25\%$ highest-entropy positions in the teacher model, eliminating $75\%$ of these forward evaluations. Our final ARM-to-DLM OPD objective is
\begin{equation}\label{finalobjective}
    \max_\theta \ \cL(\theta):= \E_{\q\sim \cD,\ \x\sim \pi_{\theta_\mathrm{old}}(\cdot\mid \q)}\frac{1}{|\cH^\mathtt{te}|}  \sum_{t\in \cH^\mathtt{te}} \min\left(
\rho_t^\theta \cdot A_t,\ \mathrm{clip}(\rho_t^\theta,1-\epsilon,1+\epsilon)\cdot A_t
\right),
\end{equation}
where $\cH^\mathtt{te}:=\{t\in[|\x|]: H_t^\mathtt{te}\in \mathrm{Top}_{25\%}(\{H_s^\mathtt{te}\}_{s=1}^{|\x|})\}$ with $H_t^\mathtt{te}:=-\sum_{x\in\cV} \pi_\mathtt{te}(x\mid \q,\x_{<t})\log\pi_\mathtt{te}(x\mid \q,\x_{<t})$ being the teacher’s token-level entropy at position $t$. The JustOPD algorithm is presented in Algorithm~\ref{alg}.

\begin{remark}\label{rmk:ar}
Following the argument of \citet{ni2026the}, the autoregressive-style sampling in JustOPD does not turn the DLM into an ARM. We use autoregressive sampling only during training, to match the sampling mode of the teacher ARM. At test time, the model retains its any-order generation capability, and its architecture, including bidirectional attention, is unchanged.
\end{remark}

\begin{remark}[Clipping Teacher-to-Student Probability Ratio]\label{rmk:ratio}
The variance of the single-sample MC estimator is governed by how fast $A_t$ grows in $u_t$, i.e., on tokens to which the teacher assigns much higher probability than the student. For reverse KL and Jensen-Shannon divergence, $A_t$ grows only logarithmically in $u_t$, so the single-sample estimator is well behaved. For the $\alpha$-divergence, $A_t$ grows as $u_t^\alpha$, interpolating between logarithmic growth as $\alpha\to0^+$ and linear growth as $\alpha\to1^-$ (see Figure \ref{fig:advantage_growth}). For forward KL, $A_t=u_t$ grows linear in $u_t$ and its second moment $\E_{x^t\sim \pi_\theta}[u_t^2]=\sum_{v}\pi_{\mathtt{te}}(v\mid \q,\x^{<t})^2/\pi_\theta(v\mid \q,\x^{<t})$ can be very large whenever the student underweights tokens favored by the teacher, which leads to high-variance gradients and unstable training. For such teacher-weighted divergences, it is therefore preferable to compute the divergence at the vocabulary level, either over the full vocabulary with clipping or over the teacher's Top-$K$ tokens. Since we focus on sampled-token distillation in the experiments, we clip $\log u_t$ to $[-5,5]$ to stabilize training for the forward KL and the $\alpha$-divergence.
% For such teacher-weighted divergences, it is therefore preferable to compute the divergence at the vocabulary level (either clipping or Top-K). Since the teacher's top-$K$ tokens typically carry most of its probability mass, the truncation error of \eqref{topk_div} is small for forward KL and for the $\alpha$-divergence with $\alpha$ close to $1$. In contrast, reverse KL weights each token by the student's probability, and the teacher's top-$K$ tokens need not cover the student's mass. In summary, single-sample MC with the advantage \eqref{corrected_adv} is preferred for reverse KL, Jensen-Shannon divergence, and the $\alpha$-divergence with small $\alpha$, while the top-$K$ estimate \eqref{topk_div} is preferred for forward KL and the $\alpha$-divergence with $\alpha$ close to $1$. 
\end{remark}

\begin{figure}[!ht]
\centering
\includegraphics[width=0.85\linewidth]{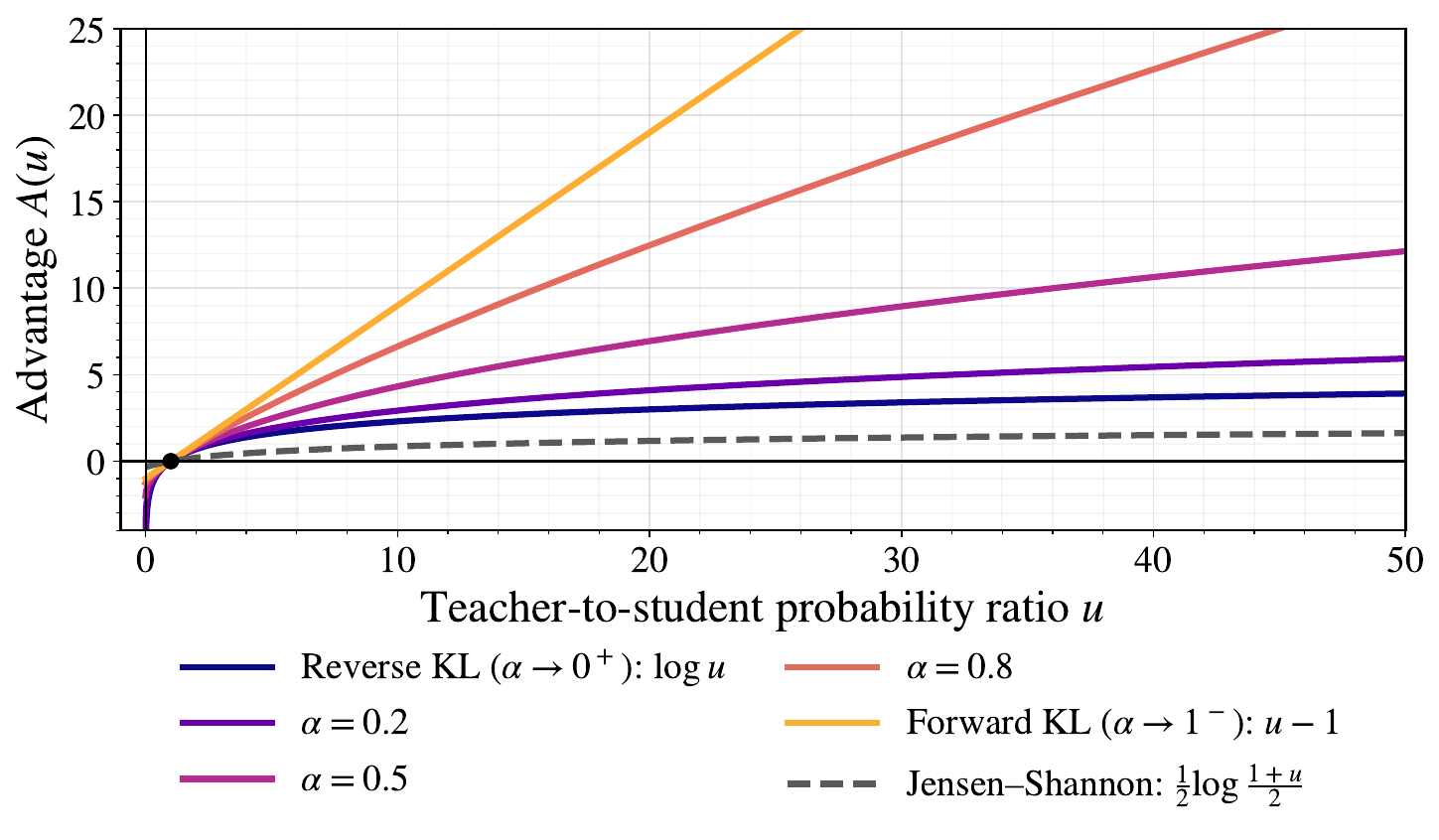}
\caption{The $\alpha$-divergence advantage interpolates between the logarithmic growth of reverse KL and the linear growth of forward KL. The Jensen-Shannon advantage also grows logarithmically.}
\label{fig:advantage_growth}
\end{figure}

\textbf{Black-Box OPD for DLMs.} When the teacher model is a black box that can only be queried without access to its logits, we consider full-vocabulary distillation with the forward KL divergence ($f(u)=u\log u$). In this case, the on-policy training objective in \eqref{topk_div} can be written without explicitly estimating the KL divergence as follows:
\begin{align}
    &\min_\theta \E_{\q\sim \cD,\ \x\sim \pi_{\theta_\mathrm{old}}(\cdot\mid \q)}\frac{1}{|\x|}\sum_{t=1}^{|\x|} D_\text{KL}\big(\pi_{\mathtt{te}}(\cdot\mid \q,\x^{<t})\,\Vert\, \pi_\theta(\cdot\mid \q,\x^{<t})\big)\notag\\
    ={}& \min_\theta \E_{\q\sim \cD,\ \x\sim \pi_{\theta_\mathrm{old}}(\cdot\mid \q)}\frac{1}{|\x|}\sum_{t=1}^{|\x|} 
    \E_{y^t\sim\pi_{\mathtt{te}}(\cdot\mid \q,\x^{<t})}
    \log\big(\pi_{\mathtt{te}}(y^t\mid \q,\x^{<t})/\pi_\theta(y^t\mid \q,\x^{<t})\big)\notag\\
        ={}& \min_\theta \E_{\q\sim \cD,\ \x\sim \pi_{\theta_\mathrm{old}}(\cdot\mid \q)}\frac{1}{|\x|}\sum_{t=1}^{|\x|} 
    -\E_{y^t\sim\pi_{\mathtt{te}}(\cdot\mid \q,\x^{<t})}
    \log \pi_\theta(y^t\mid \q,\x^{<t}) + C',\label{obj:blackbox}
\end{align}
where $C'$ is a constant that is not dependent on $\theta$. The objective~\eqref{obj:blackbox} is similar to supervised fine-tuning (SFT) objective \citep{sun2026step}, which gives us a black box OPD algorithm for DLMs by sampling tokens from the teacher (Algorithm~\ref{alg:blackbox} in Appendix~\ref{app:alg}). The inner expectation in \eqref{obj:blackbox} is approximated via the sampled tokens from teacher by Monte Carlo estimation. In addition, since we do not have access to the logits of teacher model, we cannot compute the teacher entropy as in Algorithm~\ref{alg}. To reduce the number of forward passes required for loss computation, we consider trajectory subsampling by randomly sampling a subset of time steps $\cT \subset [|\x|]$ along each trajectory $\x$ and computing the loss only at the sampled time steps \citep{zhang2026continuous}:
\begin{equation}\label{obj:blackbox_subsampling}
    \min_\theta \E_{\q\sim \cD,\ \x\sim \pi_{\theta_\mathrm{old}}(\cdot\mid \q)}\frac{1}{|\cT|}\sum_{t\in \cT} 
    -\E_{y^t\sim\pi_{\mathtt{te}}(\cdot\mid \q,\x^{<t})}
    \log \pi_\theta(y^t\mid \q,\x^{<t}).
\end{equation}

\section{Experiments}

\subsection{Setup}
The choice of teacher and student should be carefully chosen because ARMs and DLMs usually use different tokenizers. Dream-7B \citep{ye2025dream} is a full-attention diffusion language model initialized from a pretrained ARM checkpoint (Qwen2.5-7B). It inherits the Qwen tokenizer, so it shares its vocabulary with the Qwen model family. This allows token-level distillation from Qwen models without complex vocabulary alignment. We therefore use Dream-7B as the student and Qwen3-14B \citep{yang2025qwen3} as the teacher in all experiments.

\textbf{Teacher-student gap under restricted generation lengths.} OPD is expected to be most effective when there is a clear performance gap between teacher and student. DLMs are typically trained and decoded with much shorter sequence lengths than ARMs. As a result, the teacher is queried on relatively short sequences generated by the DLM (e.g., $L \leq 1024$), rather than the much longer sequences (e.g., $L=8192$) that ARMs normally use. We therefore first evaluate both models on two mathematical reasoning benchmarks, GSM8K \citep{cobbe2021training} and MATH500 \citep{lightman2024let}, with generation lengths restricted to $L\in \{256, 512, 1024\}$. We report accuracy in Table~\ref{tab:dream_qwen_acc}. Throughout this work, Dream-7B is evaluated with $L/2$ denoising steps, temperature $0.1$, top-$p$ sampling with $p=0.9$, and the entropy-based unmasking strategy, which unmasks the 2 lowest-entropy positions at each denoising step. Evaluation therefore uses the DLM’s native arbitrary-order decoding, even though training rollouts are generated left to right.

Table~\ref{tab:dream_qwen_acc} shows that the teacher–student gap depends strongly on $L$, and the two models move in opposite directions. Qwen3-14B degrades sharply at short lengths. At $L=256$ it reaches only $51.71\%$ on GSM8K and $14.40\%$ on MATH500, far below Dream-7B ($77.03\%$ and $31.80\%$). The likely cause is that Qwen3-14B is trained to produce much longer responses: its native context length is 32K tokens, compared with 2K for Dream-7B. Many of its responses are therefore truncated before reaching a final answer. We also see that Dream-7B degrades at long lengths of $L=1024$ on both tasks. Based on these observations, we use the original Qwen3-14B as the teacher at $L=512$ for GSM8K and at $L=1024$ for MATH500.

\begin{table}[!ht]
\centering
\setlength{\tabcolsep}{5pt}
\begin{tabular}{c cc cc}
\toprule
& \multicolumn{2}{c}{\textbf{GSM8K}}
& \multicolumn{2}{c}{\textbf{MATH500}} \\
\cmidrule(lr){2-3}
\cmidrule(lr){4-5}
\textbf{Accuracy}
& \textbf{Dream-7B} & \textbf{Qwen3-14B}
& \textbf{Dream-7B} & \textbf{Qwen3-14B}
\\
\midrule
$L=256$
& 77.03 & 51.71
& 31.80 & 14.40\\

$L=512$
& \cellcolor{gray!20}78.77
& \cellcolor{gray!20}93.93
& 48.20 & 48.20 \\

$L=1024$
& 56.25 & 94.84
& \cellcolor{gray!20}36.80 & \cellcolor{gray!20}70.80 \\
\bottomrule
\end{tabular}
\caption{Performance of Dream-7B and Qwen3-14B with different generation lengths $L$.}
\label{tab:dream_qwen_acc}
\end{table}

\textbf{Making the teacher effective at short lengths.} To study distillation in the short-sequence regime, where DLMs are most efficient, we fine-tune a separate Qwen3-14B model for each task. Each model is trained to produce shorter answers (e.g., fewer than $256$ tokens) while largely preserving accuracy. We denote the resulting models by Qwen3-14B-FT; fine-tuning details are provided in Appendix~\ref{FTQwen}. As shown in Table~\ref{tab:dream_qwenFT_acc}, fine-tuning raises Qwen3-14B’s accuracy at $L=256$ from $51.71\%$ to $93.48\%$ on GSM8K and from $14.40\%$ to $53.20\%$ on MATH500. Both now exceed Dream-7B by a wide margin, which makes Qwen3-14B-FT a suitable teacher at $L=256$.

\begin{table}[!ht]
\centering
\setlength{\tabcolsep}{5pt}
\begin{tabular}{c cc cc}
\toprule
& \multicolumn{2}{c}{\textbf{GSM8K}}
& \multicolumn{2}{c}{\textbf{MATH500}} \\
\cmidrule(lr){2-3}
\cmidrule(lr){4-5}
\textbf{Accuracy}
& \textbf{Dream-7B} & \textbf{Qwen3-14B-FT}
& \textbf{Dream-7B} & \textbf{Qwen3-14B-FT}
\\
\midrule
$L=256$
& \cellcolor{gray!20}77.03 & \cellcolor{gray!20}93.48
& \cellcolor{gray!20}31.80 & \cellcolor{gray!20}53.20
\\

% $L=512$
% & \cellcolor{gray!20}84.99
% & \cellcolor{gray!20}93.93
% & 48.20 & 48.20
% & 0.00 & 3.33
% &0.00 & 0.83 \\

% $L=1024$
% & 56.25 & 94.84
% & \cellcolor{gray!20}36.80 & \cellcolor{gray!20}70.80
% & \cellcolor{gray!20}0.00 & \cellcolor{gray!20}13.33
% & \cellcolor{gray!20}0.00 & \cellcolor{gray!20}4.58 \\
\bottomrule
\end{tabular}
\caption{Performance of Dream-7B and fine-tuned Qwen3-14B with a short sequence length of $L=256$.}
\label{tab:dream_qwenFT_acc}
\end{table}

\subsection{Results} 
We fine-tune Dream-7B with JustOPD for 100 training steps, evaluate every 10 steps, and report the best checkpoint. We compare $\alpha$-divergences with $\alpha \in \{0, 0.2, 0.5, 0.8, 1\}$ at $L=256$, where $\alpha=0$ recovers the reverse KL and $\alpha=1$ the forward KL, as well as the Jensen–Shannon divergence (JSD). Further training details are given in Appendix~\ref{app:training_details}.
Table~\ref{tab:results2} reports results with Qwen3-14B-FT as the teacher at $L=256$. Table~\ref{tab:results} reports results with the original Qwen3-14B as the teacher at $L=512$ (GSM8K) and $L=1024$ (MATH500), where we choose $\alpha\in \{0,0.2,0.5,0.8,1\}$ due to the time cost at longer sequences.

\begin{table}[!ht]
\centering
\setlength{\tabcolsep}{8pt}
\begin{tabular}{l cc cc cc c}
\toprule
\multicolumn{1}{c}{\textbf{Acc.} $(L=256)$}
& Dream-7B
& \textbf{$\alpha=0$} 
& \textbf{$\alpha=0.2$}
& \textbf{$\alpha=0.5$}
& \textbf{$\alpha=0.8$}
& \textbf{$\alpha=1$}
& JSD\\
\midrule
\textbf{GSM8K}
& 77.03 
& \underline{80.97}
& \textbf{81.80}
& 80.36 & 80.67 & 79.68 & 80.52 \\
\textbf{MATH500}
& 31.80
& 38.80
& 36.60
& 38.40 & \textbf{39.60}& 38.40 & \underline{39.40} \\
\bottomrule
\end{tabular}
\caption{Performance of Dream-7B before and after OPD with Qwen3-14B-FT on GSM8K and MATH500 on the short sequence length of $L=256$ across different $\alpha$-divergences and Jensen-Shannon divergence (JSD). $\alpha=0$ is the reverse KL and $\alpha=1$ is the forward KL. }
\label{tab:results2}
\end{table}

\begin{table}[!ht]
\centering
\setlength{\tabcolsep}{8pt}
\begin{tabular}{l cc cc c}
\toprule
\multicolumn{1}{c}{\textbf{Acc.}}
& Dream-7B
& \textbf{$\alpha=0$} 
& \textbf{$\alpha=0.5$}
& \textbf{$\alpha=1$} & JSD\\
\midrule
\textbf{GSM8K} ($L=512$)
& 78.77
& \textbf{84.45}
& \underline{83.24} 
& 79.23 & 81.96 \\
\textbf{MATH500} ($L=1024$)
& 36.80
& 43.80
& \underline{46.20}
& 42.80  & \textbf{48.00} \\
\bottomrule
\end{tabular}
\caption{Performance of Dream-7B before and after OPD on GSM8K and MATH500 with Qwen3-14B across different $\alpha$-divergences and JSD.  The generation length $L$ is specified for each benchmark.}
\label{tab:results}
\end{table}

\textbf{JustOPD consistently improves the student.} In every setting, JustOPD improves over the base Dream-7B, across all divergences, both benchmarks, and all three generation lengths.
These gains are measured with the DLM’s native entropy-based, arbitrary-order decoder, even though the student was trained only on left-to-right rollouts. Knowledge distilled through autoregressive rollouts therefore transfers to any-order inference, which supports Remark~\ref{rmk:ar}: JustOPD aligns training with the ARM teacher without constraining the DLM’s generation mode at test time.

\textbf{Distillation partially offsets the student’s length-dependent weaknesses.} At $L=256$, the JustOPD-trained student reaches $38.80\%$ on MATH500 with reverse KL and up to $39.60\%$ with $\alpha=0.8$. Both exceed the base Dream-7B at $L=1024\ (36.80\%)$ while generating a quarter as many tokens. At $L=1024$, the best distilled model ($48.00\%$ with JSD) matches the base model’s accuracy at $L=512\  (48.20\%)$. This suggests that a large part of the gain at long lengths comes from mitigating the degradation Dream-7B shows when asked to generate longer responses.

\textbf{Effect of the divergence.} No single divergence is best across all settings, but there are some consistent patterns. Mode-seeking divergences work well on GSM8K. In Table~\ref{tab:results}, reverse KL gives by far the largest improvement (+6.68), followed by $\alpha=0.5 \ (+4.47)$ and JSD $(+3.19)$. The forward KL barely improves on the base model $(+0.46)$. 
In Table~\ref{tab:results2}, the spread on GSM8K is much narrower ($79.68\%$ to $81.80\%$), but the forward KL is again the weakest, which matches Remark~\ref{rmk:ratio}. With sampled-token estimation, the forward-KL advantage $A_t = u_t$ is linear in $u_t$, so the gradient has high variance whenever the student underweights tokens the teacher favors. Clipping $\log u_t$ stabilizes training but may also discard useful signal. 
Divergences that interpolate between the two KL directions work well on MATH500. At $L= 1024$, JSD $(48.00\%)$ and $\alpha=0.5 \ (46.20\%)$ outperform both the forward KL $(44.60\%)$ and the reverse KL $(43.80\%)$; at $L=256$, the best results are from $\alpha=0.8 \ (39.60\%)$ and JSD $(39.40\%)$. One possible explanation is that MATH500 solutions involve more branching reasoning, making the teacher's next-token distributions more likely to assign substantial probability to multiple plausible continuations.
In this case, an objective that does not overly favor the most likely continuation may better preserve the teacher's distribution over plausible alternatives.
JSD is a robust choice, which improves every setting and gives the best MATH500 result at $L = 1024$. Its advantage grows only logarithmically in $u_t$, which keeps the single-sample estimator well behaved.

Overall, these results show that ARM-to-DLM on-policy distillation with autoregressive rollouts is effective. A stronger ARM teacher can provide useful dense supervision to a DLM student, provided the teacher’s generation length is matched to the regime in which the student operates.

% \subsection{Block-Attention DLMs}
% SDAR \citep{cheng2026sdar} adapts pretrained AR models into block diffusion models via off-policy conversion: The SDAR-Chat series (1.7B, 4B, 8B, 30B-A3B) are generated starting from the Qwen3 series (1.7B, 4B, 8B, 30B-A3B) \citep{yang2025qwen3} base models. Thus, we may consider SDAR-Chat-4B the student model and Qwen3-8B the teacher model.

\section{Conclusions and Future Works}
We studied on-policy distillation from autoregressive to diffusion language models and proposed JustOPD. JustOPD resolves the mismatch between the two generation paradigms by using left-to-right rollouts for the DLM student during training only. This lets an ARM teacher supervise every intermediate student state, while the student keeps its arbitrary-order decoding at inference. Experiments with Dream-7B as the student and Qwen3-14B as the teacher show consistent gains on GSM8K and MATH500 across $f$-divergences and generation lengths.

Several promising directions remain for future work. First, preliminary results on AIME 2024 show that JustOPD enables Dream-7B to achieve a nonzero Pass@16, whereas the base model solves no problems (Table~\ref{tab:aime}). Future work could evaluate JustOPD on more challenging reasoning benchmarks, with longer generation lengths and stronger teachers. Second, JustOPD could be extended to block-attention DLMs, such as Efficient-DLM \citep{fu2026efficientdlm} and SDAR \citep{cheng2026sdar}, whose semi-autoregressive structure may be particularly well suited to distillation from ARM teachers. Third, future work would evaluate the black-box OPD for DLMs, which would enable distillation from teachers whose internal logits are inaccessible.
Finally, the current token-level distillation approach assumes that the teacher and student share a tokenizer, limiting the range of compatible model pairs. Incorporating cross-tokenizer distillation techniques \citep{patiño2025_unlocking_on_policy_distillation_for_any_model_family} could remove this constraint and broaden JustOPD's applicability.

% Third, autoregressive rollouts require one forward pass per token, making rollout generation a major computational bottleneck. Integrating JustOPD with inference-acceleration techniques, such as KV caching and parallel decoding (e.g., Fast-dLLM \citep{wu2026fast,wu2026fastv2}), could substantially reduce training costs. 

\section*{Acknowledgment}
Tang is supported by NSF CAREER Award DMS-2538791, NSF MFAI Award DMS-2602038 and the Tang Family Assistant Professorship.
Sheng and Zhang are supported by NSF CAREER Award DMS-2538791.
This work is part of a Columbia-CityU/HK collaborative project that is supported by InnoHK Initiative, The Government of the HKSAR and the AIFT Lab.

{
        \bibliographystyle{plainnat}
        \bibliography{ref.bib}
}

\begin{appendix}
\newpage

\section{Inference Strategies of Masked Diffusion Models}\label{app:denoising}

For masked diffusion models, at a noise level $\tau\in [0,1]$, the forward transition kernel factorizes across dimensions as
$
p_{\tau|0}(\x_\tau\mid \x_0)
=
\prod_{i=1}^L p_{\tau|0}(x_\tau^i\mid x_0^i),
$
where
$$
p_{\tau|0}(x_\tau^i\mid x_0^i)=\mathrm{Cat}(x_\tau^i;\alpha_\tau \e_{x_0^i}+(1-\alpha_\tau)\e_\mathtt{m}),
$$
where $\alpha_\tau$ is the noise scheduled to satisfy $\alpha_0\approx 1$ and $\alpha_1\approx 0$. Given $\x_0\sim p_\mathrm{data}$, for $0\leq \sigma\leq \tau\leq 1$, the reverse process gradually recovers the clean data from the pure mask, with the transition probability \citep{shi2024simplified} given by
\begin{equation}\label{random_denoising}
p_{\sigma|\tau,0}(x_\sigma^i\mid \x_\tau,\x_0)=\begin{cases}
    \mathrm{Cat}(x_\sigma^i;\e_{x_\tau^i}), & x_\tau^i\neq \mathtt{m};\\
    \mathrm{Cat}(x_\sigma^i;\frac{1-\alpha_\sigma}{1-\alpha_\tau}\e_\mathtt{m}+\frac{\alpha_\sigma-\alpha_\tau}{1-\alpha_\tau}\e_{x_0^i}),& x_\tau^i=\mathtt{m}.
\end{cases}
\end{equation}
We see that conditional on $\x_0$, a masked token $x_\tau^i=\mathtt{m}$ will jump to $x_0^i$ at time $\sigma$ with probability $\frac{\alpha_\sigma-\alpha_\tau}{1-\alpha_\tau}$ and then remains the same until $\tau=0$. 

% Analogously, \citet{tang2026tweedie} derive conditional-expectation representations of score functions for non-Gaussian processes, yielding corresponding denoising score-matching objectives for continuous-state diffusion models.

At inference time, starting from a fully masked sequence $\x_0$, a clean sequence $\x$ is unmasked via an iterative trajectory $(\x_t)_{t=0}^T$ with $\x_0=(\mathtt{m},\dots,\mathtt{m})\in\{\mathtt{m}\}^L$, where each intermediate $\x_t$ is obtained from $\x_{t-1}$ by unmasking a subset of its masked positions while keeping all previously unmasked tokens fixed\footnote{We do not consider remasking for correction \citep{wang2026remasking} here.}, and $\x_T (=\x)$ is the final clean output, where $T (\leq L)\in \Z_+$ is the number of denoising steps. According to \eqref{random_denoising}, the vanilla random sampling \citep{shi2024simplified,nie2026large} is: for $t=1,\dots,T$,
\begin{itemize}
    \item Sample a set of masked positions $\mathcal{S}_t\subset \mathcal{M}(\x_{t-1}):=\{i\in[L]: x_{t-1}^i=\mathtt{m}\}$, where each $i\in\mathcal{M}(\x_{t-1})$ is included independently with probability
    \begin{equation*}
        \mathbb{P}(i\in \mathcal{S}_t)=\frac{\alpha_t-\alpha_{t-1}}{1-\alpha_{t-1}},
    \end{equation*}
    where $0=\alpha_0<\alpha_1<\dots<\alpha_T=1$ is the unmasking schedule (so that, in expectation, a fraction $\alpha_t$ of the tokens is unmasked after step $t$).
    \item For each $i\in\mathcal{S}_t$, sample $x_t^i\sim \mathrm{Cat}(\cdot;\p_\theta^{(i)}(\x_{t-1}))$ independently across positions, and set $x_t^i=x_{t-1}^i$ for all $i\notin\mathcal{S}_t$.
\end{itemize}
In contrast, adaptive sampling \citep{kim2025train} replaces the random selection of positions by a deterministic (or model-dependent) oracle $\mathcal{F}$, i.e., $\mathcal{S}_t=\mathcal{F}(\theta,\x_{t-1})\subset\mathcal{M}(\x_{t-1})$, while the second step stays the same. A common choice of $\mathcal{F}$ is the top-probability (confidence) strategy \citep{chang2022maskgit}: to unmask $K_t$ positions at step $t$ (i.e., $|\mathcal{S}_t|=K_t$), each masked position $i\in\mathcal{M}(\x_{t-1})$ is assigned the confidence score
\begin{equation*}
    c_i:=\max_{j\in\mathcal{V}}\, p_\theta^{(i)}(x^i=j\,|\,\x_{t-1})\in[0,1],
\end{equation*}
and $\mathcal{F}(\theta,\x_{t-1})=\mathrm{TopK}_{K_t}\big(\{c_i\}_{i\in\mathcal{M}(\x_{t-1})}\big)$ selects the $K_t$ most confident masked positions. Another widely used choice is the lowest-entropy strategy \citep{ye2025dream}: each masked position $i\in\mathcal{M}(\x_{t-1})$ is assigned the entropy of its predicted marginal,
\begin{equation*}
    H_i:=-\sum_{j\in\mathcal{V}} p_\theta^{(i)}(x^i=j\,|\,\x_{t-1})\,\log p_\theta^{(i)}(x^i=j\,|\,\x_{t-1})\in[0,\log|\mathcal{V}|],
\end{equation*}
and $\mathcal{F}(\theta,\x_{t-1})=\mathrm{TopK}_{K_t}\big(\{-H_i\}_{i\in\mathcal{M}(\x_{t-1})}\big)$ selects the $K_t$ masked positions with the smallest entropy. Unlike the confidence score $c_i$, which only looks at the most likely token, $H_i$ accounts for the whole predictive distribution. A position is therefore considered certain only when the probability mass is concentrated on few tokens, not merely when one token is ahead of the rest. 
The autoregressive-type sampling procedure in \eqref{ARpolicy} corresponds to the intermediate state
$$
\x_t=(\q,\x^{\leq t},\mathtt{m},\ldots,\mathtt{m}),\quad  t\in [T],
$$
with the final clean sequence given by $\x=\x_T$.

\section{Algorithms}\label{app:alg}
JustOPD algorithms for Full/TopK-Vocabulary Distillation and Black-Box Teacher Distillation are presented in Algorithms~\ref{alg:full_v} and \ref{alg:blackbox}, respectively.

\begin{algorithm}[!ht]
\caption{JustOPD for DLMs: Full/TopK-Vocabulary Distillation}\label{alg:full_v}
\begin{algorithmic}[1]
\Require ARM teacher policy $\pi_{\mathtt{te}}$, DLM student policy $\pi_\theta$, $f$-divergence clip $C$ or Topk threshold $K$, learning rate $\eta$, prompt dataset $\cD$
\State Select $\dagger \in \{\text{clip},\text{TopK}\}$
\For{each training iteration}
    \State Set $\pi_{\theta_{\text{old}}} \leftarrow \pi_\theta$
    \State Sample a prompt batch $\cB \subset \cD$
    \State Rollout buffer $\mathcal{R} \leftarrow \emptyset$
    \For{each prompt $\q \in \cB$}
        \State Sample $\x = (x^1, \cdots, x^{|\x|}) \sim \pi_{\theta_{\text{old}}}(\cdot \mid \q)$
        \For{$t = 1$ to $|\x|$}
        \If{$\dagger = \text{TopK}$}
            \State Query teacher: $Q_t=\left(\pi_{\mathtt{te}}(\cdot \mid \q,\x^{<t}), \text{TopK}(\pi_{\mathtt{te}}(\cdot \mid \q,\x^{<t}))\right),H_t^\mathtt{te}$
        \Else
         \State Query teacher: $Q_t=\pi_{\mathtt{te}}(\cdot \mid \q,\x^{<t}),H_t^\mathtt{te}$
        \EndIf

        \EndFor
        \State $\cH^\mathtt{te}\leftarrow \{t\in[|\x|]: H_t^\mathtt{te}\in \mathrm{Top}_{25\%}(\{H_s^\mathtt{te}\}_{s=1}^{|\x|})\}$
        \State Store trajectory-level data:
        $
        \left(\q, \x, \cH^\mathtt{te}, \{Q_t\}_{t\in \cH^\mathtt{te}}\right) \rightarrow \mathcal{R}
        $
    \EndFor
    \For{each mini-batch gradient step}
        \State Sample a mini-batch of prompts $\cB_{\text{mini}} \subset \mathcal{R}$
        \State Compute the loss using \eqref{topk_div}:
        \[
        \mathcal{L}(\theta) = \frac{1}{ |\cB_{\text{mini}}|} \sum_{(\q,\x) \in \cB_{\text{mini}}} \frac{1}{|\cH^\mathtt{te}|}  \sum_{t\in \cH^\mathtt{te}} D_f^{\dagger}\big(\pi_{\mathtt{te}}(\cdot\mid \q,\x^{<t})\,\Vert\, \pi_\theta(\cdot\mid \q,\x^{<t})\big)
        \]
        \State Update parameters: $\theta \leftarrow \theta - \eta \nabla_\theta \mathcal{L}(\theta)$
    \EndFor
\EndFor
\end{algorithmic}
\end{algorithm}

\begin{algorithm}[!ht]
\caption{JustOPD for DLMs: Black-Box Teacher Distillation}\label{alg:blackbox}
\begin{algorithmic}[1]
\Require ARM teacher policy $\pi_{\mathtt{te}}$, DLM student policy $\pi_\theta$, learning rate $\eta$, prompt dataset $\cD$
\For{each training iteration}
    \State Set $\pi_{\theta_{\text{old}}} \leftarrow \pi_\theta$
    \State Sample a prompt batch $\cB \subset \cD$
    \State Rollout buffer $\mathcal{R} \leftarrow \emptyset$
    \For{each prompt $\q \in \cB$}
        \State Sample $\x = (x^1, \cdots, x^{|\x|}) \sim \pi_{\theta_{\text{old}}}(\cdot \mid \q)$
        \For{$t = 1$ to $|\x|$}
         \State Sample tokens from the teacher: $Q_t=\{y^t\}\sim\pi_{\mathtt{te}}(\cdot \mid \q,\x^{<t})$
        \EndFor
        % \State $\cH^\mathtt{te}\leftarrow \{t\in[|\x|]: H_t^\mathtt{te}\in \mathrm{Top}_{25\%}(\{H_s^\mathtt{te}\}_{s=1}^{|\x|})\}$
        \State Store trajectory-level data:
        $
        \left(\q, \x, \{Q_t\}_{t\in \cH^\mathtt{te}}\right) \rightarrow \mathcal{R}
        $
    \EndFor
    \For{each mini-batch gradient step}
        \State Sample a mini-batch $\cB_{\text{mini}} \subset \mathcal{R}$ and a time set $\cT \subset [|\x|]$ for each $(\q,\x)\in \cB_\text{mini}$
        \State Compute the loss using \eqref{obj:blackbox_subsampling}:
        \[
        \mathcal{L}(\theta) = -\frac{1}{ |\cB_{\text{mini}}|} \sum_{(\q,\x) \in \cB_{\text{mini}}} \frac{1}{|\cT|}  \sum_{t\in \cT} 
    \frac{1}{|Q_t|}\sum_{y^t\in Q_t}
    \log \pi_\theta(y^t\mid \q,\x^{<t})
        \]
        \State Update parameters: $\theta \leftarrow \theta - \eta \nabla_\theta \mathcal{L}(\theta)$
    \EndFor
\EndFor
\end{algorithmic}
\end{algorithm}

\section{Proofs}
\subsection{Proof of Proposition~\ref{prop:signal_to_noise}}\label{signal_to_noise}
\textit{Proof of Proposition~\ref{prop:signal_to_noise}.}
We quantify the error incurred by the single-sample approximation in~\eqref{OPDobjective} and show that it is governed by the local behavior of $f$ at $u=1$.

Fix a prompt $\q$ and write the rollout as $\x=(x^1,\dots,x^T)\sim\pi_\theta(\cdot\mid \q)$, where $T=|\x|$ is treated as deterministic, as is the case for DLMs decoded with a fixed generation length $L$.\footnote{If $T$ is determined by an end-of-sequence stopping rule, the arguments below apply to the stopped martingale, with minor modifications to the $1/T$ normalization that we omit.}
Let $\mathcal{F}_t:=\sigma(x^1,\dots,x^t)$ denote the natural filtration, with $\mathcal{F}_0$ trivial. For notational simplicity, for each position~$t$, let the student and teacher next-token distributions be
\[
Q_t(\cdot):=\pi_\theta(\cdot\mid \q,\x^{<t})\quad \text{and}\quad P_t(\cdot):=\pi_{\mathrm{te}}(\cdot\mid \q,\x^{<t})
\]
are $\mathcal{F}_{t-1}$-measurable, and $x^t\mid\mathcal{F}_{t-1}\sim Q_t$. We write
\[
u_t(x):=\frac{P_t(x)}{Q_t(x)}=1+\Delta_t(x)
\quad \text{and}\quad
D_{f,t}:=D_f(P_t\,\|\,Q_t)=\mathbb{E}_{x'\sim Q_t}\big[f(u_t(x'))\big].
\]
Thus $D_{f,t}$ is the exact inner expectation in~\eqref{OPDobjective}, while $f(u_t(x^t))$ is its single-sample estimate. The token-level error is
$$
\xi_t:=f\big(u_t(x^t)\big)-D_{f,t}
$$
and trajectory-level approximation error is 
\begin{equation}\label{e_T}
e_T:=\frac{1}{T}\sum_{t=1}^{T}\xi_t=\frac{M_T}{T},
\end{equation}
where we write $ M_T:=\sum_{t=1}^{T}\xi_t$. Because $D_{f,t}$ is $\mathcal{F}_{t-1}$-measurable and $x^t\mid\mathcal{F}_{t-1}\sim Q_t$, we have $\mathbb{E}[\xi_t\mid\mathcal{F}_{t-1}]=0$. Hence, $(\xi_t)_{t\ge1}$ is a martingale difference sequence and $(M_t)_{t\ge 0}$ is a martingale with respect to $(\mathcal{F}_t)_{t\ge 0}$. This structure holds without any independence across positions: the tokens $x^t$ are strongly dependent through their prefixes, but each error $\xi_t$ is centered conditionally on the past. An immediate consequence is $\mathbb{E}[e_T]=0$, so the single-sample approximation introduces no bias in the objective value.

Let $\sigma_t^2:=\mathbb{E}[\xi_t^2\mid\mathcal{F}_{t-1}]=\operatorname{Var}_{x'\sim Q_t}\big[f(u_t(x'))\big]$ denote the conditional variance of the single-sample estimate. Since martingale increments are orthogonal, i.e., $\mathbb{E}[\xi_s\xi_t]=\mathbb{E}\big[\xi_s\,\mathbb{E}[\xi_t\mid\mathcal{F}_{t-1}]\big]=0$ for $s<t$, by the definition of $e_T$ in \eqref{e_T}, we have
\begin{equation}\label{eq:mse}
\sqrt{\mathbb{E}[e_T^2]}=\frac{1}{T}\sqrt{\mathbb{E}\Big[\sum_{t=1}^{T}\sigma_t^2\Big]}.
\end{equation}
In particular, if $\mathbb{E}\big[\sum_{t=1}^{T}\sigma_t^2\big]\sim\bar\sigma^2T$, then $\sqrt{\mathbb{E}[e_T^2]}\sim\bar\sigma/\sqrt{T}$.

Consider a family of rollouts indexed by $T\to\infty$ (a martingale difference sequence). Suppose that $\frac{1}{T}\sum_{t=1}^{T}\sigma_t^2\xrightarrow{\mathrm{p}}\bar\sigma^2$  and that the conditional Lindeberg condition holds, i.e., for every $\eta>0$,
\[
\frac{1}{T}\sum_{t=1}^{T}\mathbb{E}\big[\xi_t^2\,\mathbf{1}\{|\xi_t|>\eta\sqrt{T}\}\,\big|\,\mathcal{F}_{t-1}\big]\xrightarrow{\mathrm{p}} 0,
\]
where $\mathbf{1}\{\cdot\}$ is the indicator function. 
Then the martingale central limit theorem \citep{hall1980martingale} yields $\sqrt{T}\,e_T\xrightarrow{\mathrm{d}}\mathcal{N}(0,\bar\sigma^2)$, i.e., $e_T\approx\mathcal{N}(0,\bar\sigma^2/T)$ for long trajectories. In addition, if $|\xi_t|\le B$ for all $t$, the Azuma-Hoeffding inequality \citep{hoeffding1963probability,azuma1967weighted} yields
$$
\mathbb{P}\big(|e_T|\ge\eta\big)\le 2\exp\Big(-\frac{T\eta^2}{2B^2}\Big),\quad \forall \eta>0.
$$
A sharper variance-sensitive bound follows from Freedman's inequality \citep{freedman1975tail}: for all $\eta,v>0$, it holds that
$$
\mathbb{P}\Big(|e_T|\ge\eta,\ \frac{1}{T}\sum_{t=1}^{T}\sigma_t^2\le v\Big)\le 2\exp\Big(-\frac{T\eta^2}{2\,(v+B\eta/3)}\Big).
$$
Hence, a single trajectory of length $T$ approximates the average exact token-level $f$-divergence, with the probability of an error of size $\eta$ decaying exponentially in $T$. These statements, however, are only meaningful relative to the magnitude of the quantity being estimated: in the regime of interest, where the student is already close to the teacher, both $D_{f,t}$ and $\sigma_t$ are small, and the relevant question is how they scale relative to each other. This is determined by the behavior of $f$ near $u=1$, which we analyze next.

We analyze each position $t$ conditionally on its prefix and aggregate afterwards. First observe that
\[
\mathbb{E}_{Q_t}[\Delta_t]=\sum_{x}P_t(x)-\sum_{x}Q_t(x)=0.
\]
Assume that $|\Delta_t(x)|\le\gamma<1$ for all $x$ in the support of $Q_t$, and that $f$ is four times continuously differentiable on a neighborhood of $[1-\gamma,1+\gamma]$ with $f''(1)>0$. Since $|\Delta_t|\le\gamma$, we have $\mathbb{E}_{Q_t}|\Delta_t|^k\le\gamma^{k-2}\chi_t^2$ for all $k\ge2$, where $\chi_t^2:=\chi^2(P_t\,\|\,Q_t)=\mathbb{E}_{Q_t}[\Delta_t^2]$. Expanding $f(1+\Delta)=f'(1)\Delta+\tfrac12 f''(1)\Delta^2+\tfrac16 f'''(1)\Delta^3+O(\Delta^4)$ and taking expectations under $Q_t$, the linear term vanishes and
\begin{equation}\label{eq:local-signal}
D_{f,t}=\tfrac12 f''(1)\,\chi_t^2+\tfrac16 f'''(1)\,\mathbb{E}_{Q_t}[\Delta_t^3]+O(\gamma^2\chi_t^2)
=\tfrac12 f''(1)\,\chi_t^2\,\big(1+O(\gamma)\big).
\end{equation}
 Thus, locally, every $f$-divergence is a multiple of the $\chi^2$-divergence, and in particular, $D_{f,t}=O(\gamma^2)$.

The single-sample error is:
\[
\xi_t=f'(1)\,\Delta_t(x^t)+\tfrac12 f''(1)\big(\Delta_t(x^t)^2-\chi_t^2\big)+O(\gamma^3).
\]
Squaring, taking conditional expectations, and bounding every remaining term by a constant multiple of $\mathbb{E}_{Q_t}|\Delta_t|^k\le\gamma^{k-2}\chi_t^2$ with $k\ge4$, we obtain
\begin{equation}\label{eq:local-noise}
\sigma_t^2=f'(1)^2\,\chi_t^2+f'(1)f''(1)\,\mathbb{E}_{Q_t}[\Delta_t^3]+O(\gamma^2\chi_t^2).
\end{equation}
Consequently, whenever $f'(1)\neq0$, the noise is of first order, $\sigma_t=|f'(1)|\,\chi_t\,(1+O(\gamma))$, whereas the signal~\eqref{eq:local-signal} is of second order. Therefore, the local noise-to-signal ratio is
\begin{equation*}
\frac{\sigma_t}{D_{f,t}}=\frac{2\,|f'(1)|}{f''(1)\,\chi_t}\,\big(1+O(\gamma)\big),
\end{equation*}
which diverges as the student approaches the teacher. In other words, the better the student already matches the teacher, the less informative each single-sample estimate becomes relative to the quantity it estimates. The source of this deficiency is transparent: the leading term $f'(1)\Delta_t(x^t)$ has zero conditional mean, so it contributes nothing to $D_{f,t}$ while dominating the variance.

The centralized generator $\tilde f(u)=f(u)-f(1)-f'(1)(u-1)$ removes precisely the first-order noise. Because $\mathbb{E}_{Q_t}[u_t]=1$, it leaves the signal unchanged, $D_{\tilde f,t}=D_{f,t}$, while modifying the estimator to
\[
\tilde f\big(u_t(x^t)\big)=f\big(u_t(x^t)\big)-f'(1)\,\Delta_t(x^t).
\]
Since $\tilde f(1)=\tilde f'(1)=0$ and $\tilde f''=f''$, the first two terms of~\eqref{eq:local-noise} vanish for $\tilde f$, and carrying the expansion one order further yields
\begin{equation*}
\tilde\sigma_t^2=\tfrac14 f''(1)^2\operatorname{Var}_{Q_t}\!\big(\Delta_t^2\big)+O(\gamma^3\chi_t^2)
\le\tfrac14 f''(1)^2\,\gamma^2\chi_t^2\,\big(1+O(\gamma)\big),
\end{equation*}
so we have
\begin{equation*}
\frac{\tilde\sigma_t}{D_{\tilde f,t}}=\frac{\sqrt{\operatorname{Var}_{Q_t}(\Delta_t^2)}}{\mathbb{E}_{Q_t}[\Delta_t^2]}\,\big(1+O(\gamma)\big)\le\frac{\gamma}{\chi_t}\,\big(1+O(\gamma)\big).
\end{equation*}
The leading factor is the coefficient of variation of $\Delta_t^2$ under $Q_t$, a scale-free quantity that does not degenerate as the teacher and student approach each other. In the non-degenerate regime $\chi_t\sim\gamma$, in which the teacher and student are $\Theta(\gamma)$ apart, the local orders of magnitude are summarized below.
\begin{table}[!ht]
\centering
\renewcommand{\arraystretch}{1.0}
\setlength{\tabcolsep}{16pt}
\begin{tabular}{lcc}
\toprule
 & $f$ & $\tilde f$ \\
\midrule
Signal $D_{f,t}$
    & $\Theta(\gamma^2)$
    & $\Theta(\gamma^2)$ \\[2pt]
Single-sample variance $\sigma_t^2$
    & $\Theta(\gamma^2)$
    & $O(\gamma^4)$ \\[2pt]
Single-sample standard deviation $\sigma_t$
    & $\Theta(\gamma)$
    & $O(\gamma^2)$ \\[2pt]
Noise-to-signal ratio $\sigma_t/D_{f,t}$
    & $\Theta(1/\gamma)$
    & $O(1)$ \\
\bottomrule
\end{tabular}
\caption{Comparison of signal and noise scaling for $f$ and $\tilde f$.}
\label{tab:signal_noise_comparison}
\end{table}

The condition $|\Delta_t|\le\gamma$ may be too strong when the teacher assigns much larger probability than the student to a few tokens of small $Q_t$-mass. However, the analysis above uses $\gamma$ only through moment bounds of the form $\mathbb{E}_{Q_t}|\Delta_t|^k\lesssim\gamma^{k-2}\chi_t^2$, so it suffices to control $\chi_t$ together with the higher moments of $\Delta_t$. Specifically, suppose that for some $\delta\in(0,1]$,
\begin{enumerate}
\item[(i)] the second-order Taylor remainder of $f$ satisfies $\big|f(1+\Delta)-f'(1)\Delta-\tfrac12 f''(1)\Delta^2\big|\le K|\Delta|^{2+\delta}$ on the range of $\Delta_t$;
\item[(ii)] the moments of $\Delta_t$ are comparable to its second moment, i.e., $\mathbb{E}_{Q_t}|\Delta_t|^{k}\le C_k\,(\chi_t^2)^{k/2}$ for $k\in[2,4+2\delta]$.
\end{enumerate}
Then the same Taylor argument yields $D_{f,t}=\tfrac12 f''(1)\chi_t^2\,(1+O(\chi_t^{\delta}))$, $\sigma_t^2=f'(1)^2\chi_t^2\,(1+O(\chi_t))$ and $\tilde\sigma_t^2=O(\chi_t^4)$, so all of the conclusions above continue to hold with $\gamma$ replaced by $\chi_t$. Condition~(i) holds for the forward KL. For the reverse KL, it holds whenever $u_t$ is bounded away from zero, i.e., whenever the teacher does not assign vanishing probability to tokens that the student samples with non-negligible probability. $\hfill\square$

\subsection{Proof of Proposition~\ref{prop:corrected_adv}}\label{app:prop2}

\textit{Proof of Proposition~\ref{prop:corrected_adv}.}
(i) Since $P_t$ does not depend on $\theta$, we have $\nabla_\theta u_{\theta,t}=-u_{\theta,t}\nabla_\theta\log Q_{\theta,t}$. Differentiating both the sampling probability and the ratio, which is valid term by term because the support is finite, and using $\nabla_\theta Q_{\theta,t}=Q_{\theta,t}\nabla_\theta\log Q_{\theta,t}$, gives
\begin{align*}
 \nabla_\theta D_{f,t}(\theta)
 &=\sum_x\left[\nabla_\theta Q_{\theta,t}(x)\,f(u_{\theta,t}(x))
    +Q_{\theta,t}(x)f'(u_{\theta,t}(x))\nabla_\theta u_{\theta,t}(x)\right] \\
 &=\mathbb{E}_{Q_{\theta,t}}\!\left[
       \bigl(f(u_{\theta,t})-u_{\theta,t} f'(u_{\theta,t})\bigr)
       \nabla_\theta\log Q_{\theta,t}(x)\right]
  =-\mathbb{E}_{Q_{\theta,t}}\!\left[A_f(u_{\theta,t})\,\nabla_\theta\log Q_{\theta,t}(x)\right].
\end{align*}

(ii) Since $\mathbb{E}_{Q_{\theta,t}}[u_{\theta,t}]=\sum_x P_t(x)=1$ and $f(1)=0$,
\[
D_{\tilde f,t}(\theta)=\mathbb{E}_{Q_{\theta,t}}[\tilde f(u_{\theta,t})]=\mathbb{E}_{Q_{\theta,t}} [f(u_{\theta,t})]-f(1)-f'(1)\big(\mathbb{E}_{Q_{\theta,t}}[u_{\theta,t}]-1\big)=\mathbb{E}_{Q_{\theta,t}}[f(u_{\theta,t})]=D_{f,t}(\theta).
\]
Moreover, $\tilde f'(u)=f'(u)-f'(1)$, so
\[
A_{\tilde f}(u)=-f(u)+f(1)+f'(1)(u-1)+u\big(f'(u)-f'(1)\big)=A_f(u)+f(1)-f'(1)=A_f(u)-f'(1).
\]
$\hfill\square$

\section{Experiment Details}

\subsection{Details on Fine-Tuning Qwen3-14B}\label{FTQwen}

\begin{algorithm}[!ht]
\caption{Length-Aware Fine-Tuning of Qwen3-14B}
\label{alg:qwen_length}
\begin{algorithmic}[1]
\Require Policy $\pi_\phi$, dataset $\cD$, correctness evaluator $c$,
group size $K$, length cap $L$
\For{each training iteration}
    \State Sample a prompt batch $\cB\subset\cD$
    \State Rollout buffer $\mathcal{R}\leftarrow\emptyset$
    \For{each prompt $\q\in\cB$}
        \State Construct $\tilde{\q}$ with the concise-answer instruction
        \State Sample $K$ responses $\{\x_i\}_{i=1}^{K}$ from $\pi_\phi(\cdot\mid\tilde{\q})$, each capped at $L$ tokens
        \For{$i=1$ to $K$}
            \State Set $\ell_i=|\x_i|$ and $c_i=c(\q,\x_i)$
            \State Compute reward $r_i$ using \eqref{eq:correct_reward}.
        \EndFor
        \For{$i=1$ to $K$}
            \State Compute stop-gradient group-normalized advantage $A_i$
            \State Store $(\tilde{\q},\x_i,A_i)\rightarrow\mathcal{R}$
        \EndFor
    \EndFor
    \State Compute the loss:
    \[
    \mathcal{L}(\phi)=
    -\frac{1}{|\mathcal{R}|L}
    \sum_{(\tilde{\q},\x_i,A_i)\in\mathcal{R}}
    A_i\cdot\sum_{t=1}^{|\x_i|}
    \log\pi_\phi(x_i^t\mid\tilde{\q},\x_i^{<t}).
    \]
    \State Update parameters: $\phi \leftarrow \phi - \eta \nabla_\phi \mathcal{L}(\phi)$
\EndFor
\end{algorithmic}
\end{algorithm}

To obtain teachers that remain effective under a short generation budget (e.g. $L=256$), we fine-tune Qwen3-14B in advance for GSM8K and MATH using a length-aware, GRPO-style objective (see Algorithm~\ref{alg:qwen_length}). We use the model's chat template with thinking disabled and append an instruction asking for concise step-by-step reasoning, without restating the problem. During training, responses are sampled in the auto-regressive way with temperature $T = 1.0$, top-$p$ sampling with $p=0.95$.

We reward correct answers with a small brevity bonus that is not applied to incorrect ones. Specifically, let $c_i$ denote the correctness reward returned by the benchmark-specific evaluator for response $i$, with $c_i>0$ indicating a correct answer. Let $\ell_i$ denote its generated length, including the first end-of-sequence token when present. We assign
\begin{equation}\label{eq:correct_reward}
r_i=c_i+\lambda_{\mathrm{len}}\,
\mathbf{1}_{\{c_i>0\}}\left(1-\frac{\ell_i}{L_{\max}}\right).
\end{equation}
Shorter correct responses receive larger rewards, while incorrect responses receive no brevity bonus. In Algorithm~\ref{alg:qwen_length}, we set the brevity coefficient
$\lambda=0.01$ and the learning rate $\eta=5\times10^{-5}$.
Group-normalized advantages $A_i$ are divided by the within-group sample
standard deviation plus the numerical stabilizer $\epsilon=10^{-4}$.
The advantages are held fixed during gradient computation. Adapter checkpoints are saved every $100$ steps. The resulting task-specific teachers are denoted by Qwen3-14B-FT.

\subsection{Training Details}\label{app:training_details}

We use Dream-v0-Instruct-7B~\citep{ye2025dream} as the DLM student and Qwen3-14B~\citep{yang2025qwen3} as the ARM teacher. All experiments are conducted on on 7 NVIDIA L40S GPUs. During distillation the teacher stays frozen and is queried in non-thinking mode. Both models receive the same problem through their own chat templates. The student is trained with LoRA (rank $r=128$, $\alpha=64$, dropout $0.05$). At each step, every GPU draws one prompt and generates 8 on-policy rollouts from the student, giving 7 prompts and 56 rollouts per optimizer step. Rollouts have  sampling temperature $1.0$, and are generated in left-to-right order, unmasking one token per forward pass and stopping once every sequence has emitted an end-of-sequence token. The PPO clipping parameter is $\epsilon=0.2$. For the reverse KL training case, we add a forward-KL term with coefficient $0.1$, computed on the teacher's top-16 tokens and applied only at positions where the teacher entropy exceeds $\tau=0.8$, following the work of \citet{jin2026entropyaware}. We optimize with AdamW ($\beta_1=0.9$, $\beta_2=0.999$, $\epsilon=10^{-8}$, no weight decay) at a constant learning rate of $5\times10^{-5}$.

\subsection{Additional Experiment Results}
We present preliminary results of JustOPD on AIME 2024 \citep{aime24} in Table~\ref{tab:aime}, using reverse KL and training on all historical AIME problems before evaluating on AIME 2024. We observe that AIME 2024 is sufficiently challenging for Dream-7B that it fails to produce a single correct answer, even with a generation length of $1024$ tokens (see Table~\ref{tab:aime_eval}). Future work will focus on improving the teacher model's performance on AIME under shorter generation-length constraints, enabling more effective knowledge transfer to DLM students within their limited context budgets.

\begin{table}[!ht]
\centering
\setlength{\tabcolsep}{5pt}
\begin{tabular}{c cc cc cccc}
\toprule
& \multicolumn{2}{c}{\textbf{GSM8K}}
& \multicolumn{2}{c}{\textbf{MATH500}}
& \multicolumn{4}{c}{\textbf{AIME 24}} \\
\cmidrule(lr){2-3}
\cmidrule(lr){4-5}
\cmidrule(lr){6-9}
\textbf{Gen. Length}
& \multicolumn{2}{c}{\textbf{Acc.}}
& \multicolumn{2}{c}{\textbf{Acc.}} 
& \multicolumn{2}{c}{\textbf{Pass@16}}
& \multicolumn{2}{c}{\textbf{Avg@16}} \\
\cmidrule(lr){2-3}
\cmidrule(lr){4-5}
\cmidrule(lr){6-7}
\cmidrule(lr){8-9}
& \textbf{Stu.} & \textbf{Tr.}
& \textbf{Stu.} & \textbf{Tr.}
& \textbf{Stu.} & \textbf{Tr.}
& \textbf{Stu.} & \textbf{Tr.} \\
\midrule
$L=256$
& 77.03 & 51.71
& 31.80 & 14.40
& 0.00 & 0.00
& 0.00 & 0.00 \\

$L=512$
& \cellcolor{gray!20}78.77
& \cellcolor{gray!20}93.93
& 48.20 & 48.20
& 0.00 & 3.33
&0.00 & 0.83 \\

$L=1024$
& 56.25 & 94.84
& \cellcolor{gray!20}36.80 & \cellcolor{gray!20}70.80
& \cellcolor{gray!20}0.00 & \cellcolor{gray!20}13.33
& \cellcolor{gray!20}0.00 & \cellcolor{gray!20}4.58 \\
\bottomrule
\end{tabular}
\caption{Performance of Dream-7B and Qwen3-14B with different generation lengths $L$. Stu. and Tr. denote the base Dream-7B (student) and Qwen3-14B (teacher) models, respectively.}\label{tab:aime_eval}
\end{table}

% \begin{table}[!ht]
% \centering
% \setlength{\tabcolsep}{5pt}
% \begin{tabular}{c cc cc cccc}
% \toprule
% & \multicolumn{2}{c}{\textbf{GSM8K}}
% & \multicolumn{2}{c}{\textbf{MATH500}}
% & \multicolumn{4}{c}{\textbf{AIME 24}} \\
% \cmidrule(lr){2-3}
% \cmidrule(lr){4-5}
% \cmidrule(lr){6-9}
% \textbf{Gen. Length}
% & \multicolumn{2}{c}{\textbf{Acc.}}
% & \multicolumn{2}{c}{\textbf{Acc.}} 
% & \multicolumn{2}{c}{\textbf{Pass@16}}
% & \multicolumn{2}{c}{\textbf{Avg@16}} \\
% \cmidrule(lr){2-3}
% \cmidrule(lr){4-5}
% \cmidrule(lr){6-7}
% \cmidrule(lr){8-9}
% & \textbf{Stu.} & \textbf{Tr.}
% & \textbf{Stu.} & \textbf{Tr.}
% & \textbf{Stu.} & \textbf{Tr.}
% & \textbf{Stu.} & \textbf{Tr.} \\
% \midrule
% $L=256$
% & \cellcolor{gray!20}77.03 & \cellcolor{gray!20}93.48
% & \cellcolor{gray!20}31.80 & \cellcolor{gray!20}53.20
% & 0.00 & 
% & 0.00 &  \\

% % $L=512$
% % & \cellcolor{gray!20}84.99
% % & \cellcolor{gray!20}93.93
% % & 48.20 & 48.20
% % & 0.00 & 3.33
% % &0.00 & 0.83 \\

% % $L=1024$
% % & 56.25 & 94.84
% % & \cellcolor{gray!20}36.80 & \cellcolor{gray!20}70.80
% % & \cellcolor{gray!20}0.00 & \cellcolor{gray!20}13.33
% % & \cellcolor{gray!20}0.00 & \cellcolor{gray!20}4.58 \\
% \bottomrule
% \end{tabular}
% \caption{Performance of Dream-7B and fine-tuned Qwen3-14B with a short sequence length of $L=256$. Stu. and FT-Tr. denote the base Dream-7B (student) and Fine-Tuned Qwen3-14B (teacher) models, respectively.}
% \end{table}

\begin{table}[!ht]
\centering
\setlength{\tabcolsep}{8pt}
\begin{tabular}{l cc cc}
\toprule
& \textbf{GSM8K} 
& \textbf{MATH500} 
& \multicolumn{2}{c}{\textbf{AIME 24}} \\
\cmidrule(lr){2-2}
\cmidrule(lr){3-3}
\cmidrule(lr){4-5}
\textbf{Model} 
& \textbf{Acc.} ($L=512$)
& \textbf{Acc.} ($L=1024$)
& \textbf{Pass@16} ($L=1024$)
& \textbf{Avg@16} ($L=1024$) \\
\midrule
Dream-7B
& 78.77 
& 36.80
& 0.00
& 0.00 \\
+JustOPD
& \textbf{84.45}
& \textbf{43.80}
& \textbf{3.33}
& \textbf{0.21} \\
\bottomrule
\end{tabular}
\caption{Performance of Dream-7B before and after OPD on GSM8K, MATH500, and AIME 24. 
The generation length $L$ is specified for each benchmark.}\label{tab:aime}
\end{table}

\begin{comment}
\begin{table}[!ht]
\centering
\setlength{\tabcolsep}{8pt}
\begin{tabular}{l cc cc}
\toprule
& \textbf{GSM8K} 
& \textbf{MATH500} 
& \multicolumn{2}{c}{\textbf{AIME 24}} \\
\cmidrule(lr){2-2}
\cmidrule(lr){3-3}
\cmidrule(lr){4-5}
\textbf{Model} 
& \textbf{Acc.} ($L=256$)
& \textbf{Acc.} ($L=256$)
& \textbf{Pass@16} ($L=256$)
& \textbf{Avg@16} ($L=256$) \\
\midrule
Dream-7B
& 77.03
& 31.80
& 0.00
& 0.00 \\
+JustOPD (FT)
& \textbf{80.97}
& \textbf{38.80}
& \textbf{}
& \textbf{} \\
\bottomrule
\end{tabular}
\caption{Performance of Dream-7B before and after OPD with fine-tuned Qwen3-14B on GSM8K, MATH500, and AIME 24 on the short sequence length of $L=256$.}
\end{table}

 \end{comment}
\end{appendix}

\end{document}

%% file: mysty.tex
\usepackage[colorlinks=true, linkcolor=red, citecolor=blue, urlcolor=blue]{hyperref}
\usepackage[round]{natbib}

\usepackage{algorithm}
\usepackage{algorithmicx}
\usepackage{algpseudocode}
\usepackage[utf8]{inputenc} % allow utf-8 input
\usepackage[T1]{fontenc}    % use 8-bit T1 fonts
\usepackage[english]{babel}

\usepackage{amsfonts}
\usepackage{nicefrac}       % compact symbols for 1/2, etc.
\usepackage{microtype}      % microtypography
\usepackage{lmodern}
\usepackage{amssymb,amsmath,amsthm}
\usepackage{stmaryrd}
\usepackage{bbm,bm}
\usepackage{latexsym}
\usepackage{xcolor}
\usepackage[table]{xcolor}

\usepackage{enumerate}
\usepackage{enumitem}
\usepackage{verbatim}
\usepackage{booktabs}       % professional-quality tables
\usepackage{multirow}
\usepackage[most]{tcolorbox}
\usepackage{caption}
\usepackage{subcaption}

\usepackage{url}            % simple URL typesetting
\usepackage[colorlinks=true]{hyperref}
\usepackage[capitalise,noabbrev]{cleveref}
\usepackage{caption}

\usepackage{parskip}
\usepackage{geometry}
\usepackage[short]{optidef}
\usepackage{algorithm}

\usepackage{thmtools}

\declaretheorem{proposition}

\declaretheoremstyle[qed=$\square$]{definitionwithend}

\declaretheorem[style=definitionwithend]{remark}

\crefname{assumption}{Assumption}{Assumptions}
\crefname{fact}{Fact}{Facts}
\crefname{question}{Question}{Questions}
\crefname{observation}{Observation}{Observations} 

\definecolor{gold}{rgb}{0.85,0.65,0}

\let\emptyset\varnothing

\def\Z{{\mathbb{Z}}}

\def\cB{{\cal B}}

\def\cD{{\cal D}}

\def\cH{{\cal H}}

\def\cL{{\cal L}}

\def\cS{{\cal S}}
\def\cT{{\cal T}}

\def\cV{{\cal V}}

\def\e{{\mathbf e}}

\def\p{{\mathbf p}}
\def\q{{\mathbf q}}

\def\x{{\mathbf x}}
\def\y{{\mathbf y}}

\DeclareMathOperator*{\E}{\mathbb{E}}

\usepackage{soul}

\usepackage{amsmath,amssymb,amsthm}

\usepackage{amsthm}

%% file: ref.bib
@inproceedings{
ni2026the,
title={The Flexibility Trap: Rethinking the Value of Arbitrary Order in Diffusion Language Models},
author={Zanlin Ni and Shenzhi Wang and Yang Yue and Tianyu Yu and Weilin Zhao and Yeguo Hua and Tianyi Chen and Jun Song and Cheng Yu and Bo Zheng and Gao Huang},
booktitle={Forty-third International Conference on Machine Learning},
year={2026},
url={https://openreview.net/forum?id=kpgURPRMGf}
}

@inproceedings{
jin2026entropyaware,
title={Entropy-Aware On-Policy Distillation of Language Models},
author={Woogyeol Jin and Taywon Min and Yongjin Yang and Dennis Wei and Yi Zhou and Swanand Ravindra Kadhe and Nathalie Baracaldo and Kimin Lee},
booktitle={Forty-third International Conference on Machine Learning},
year={2026},
url={https://openreview.net/forum?id=J5i09faOOf}
}

@article{zhao2026opd+,
  title={Opd+: Rethinking the advantage design for on-policy distillation},
  author={Zhao, Hanyang and Chen, Haoxian and Lin, Han and Winata, Genta Indra and Yao, David and Tang, Wenpin},
  journal={arXiv preprint arXiv:2606.01039},
  year={2026}
}

@inproceedings{cheng2026sdar,
  title={Sdar: A synergistic diffusion-autoregression paradigm for scalable sequence generation},
  author={Cheng, Shuang and Bian, Yihan and Liu, Dawei and Jiang, Yuhua and Liu, Yihao and Zhang, Linfeng and Yao, Qian and Tian, Zhongbo and Wang, Wenhai and Guo, Qipeng and others},
  booktitle={Findings of the Association for Computational Linguistics: ACL 2026},
  pages={22058--22075},
  year={2026}
}

@article{yang2025qwen3,
  title={Qwen3 technical report},
  author={Yang, An and Li, Anfeng and Yang, Baosong and Zhang, Beichen and Hui, Binyuan and Zheng, Bo and Yu, Bowen and Gao, Chang and Huang, Chengen and Lv, Chenxu and others},
  journal={arXiv preprint arXiv:2505.09388},
  year={2025}
}

@article{zhao2025diffpo,
  title={Diffpo: Training diffusion llms to reason fast and furious via reinforcement learning},
  author={Zhao, Hanyang and Liang, Dawen and Tang, Wenpin and Yao, David and Kallus, Nathan},
  journal={arXiv preprint arXiv:2510.02212},
  year={2025}
}

@article{ye2025dream,
  title={Dream 7b: Diffusion large language models},
  author={Ye, Jiacheng and Xie, Zhihui and Zheng, Lin and Gao, Jiahui and Wu, Zirui and Jiang, Xin and Li, Zhenguo and Kong, Lingpeng},
  journal={arXiv preprint arXiv:2508.15487},
  year={2025}
}

@book{hall1980martingale,
  title={Martingale limit theory and its application},
  author={Hall, Peter and Heyde, Christopher C and Birnbaum, ZW and Lukacs, E},
  volume={12},
  year={1980},
  publisher={Elsevier}
}

@article{hoeffding1963probability,
  title={Probability inequalities for sums of bounded random variables},
  author={Hoeffding, Wassily},
  journal={Journal of the American statistical association},
  volume={58},
  number={301},
  pages={13--30},
  year={1963},
  publisher={Taylor \& Francis}
}

@article{azuma1967weighted,
  title={Weighted sums of certain dependent random variables},
  author={Azuma, Kazuoki},
  journal={Tohoku Mathematical Journal, Second Series},
  volume={19},
  number={3},
  pages={357--367},
  year={1967},
  publisher={Mathematical Institute, Tohoku University}
}

@article{freedman1975tail,
  title={On tail probabilities for martingales},
  author={Freedman, David A},
  journal={the Annals of Probability},
  pages={100--118},
  year={1975},
  publisher={JSTOR}
}

@article{ren2026trace,
  title={Trace-Based On-Policy Distillation for Masked Diffusion Language Models},
  author={Ren, Haolin and Huang, Ziyang and Yuan, Chenhao and Zhao, Jun and Liu, Kang},
  journal={arXiv preprint arXiv:2607.16872},
  year={2026}
}

@article{nie2026large,
  title={Large language diffusion models},
  author={Nie, Shen and Zhu, Fengqi and You, Zebin and Zhang, Xiaolu and Ou, Jingyang and Hu, Jun and Zhou, Jun and Lin, Yankai and Wen, Ji-Rong and Li, Chongxuan},
  journal={Advances in Neural Information Processing Systems},
  volume={38},
  pages={50608--50646},
  year={2026}
}

@article{shi2024simplified,
  title={Simplified and generalized masked diffusion for discrete data},
  author={Shi, Jiaxin and Han, Kehang and Wang, Zhe and Doucet, Arnaud and Titsias, Michalis},
  journal={Advances in neural information processing systems},
  volume={37},
  pages={103131--103167},
  year={2024}
}

@article{sahoo2024simple,
  title={Simple and effective masked diffusion language models},
  author={Sahoo, Subham S and Arriola, Marianne and Schiff, Yair and Gokaslan, Aaron and Marroquin, Edgar and Chiu, Justin T and Rush, Alexander and Kuleshov, Volodymyr},
  journal={Advances in Neural Information Processing Systems},
  volume={37},
  pages={130136--130184},
  year={2024}
}

@inproceedings{wang2026revolutionizing,
  title={Revolutionizing reinforcement learning framework for diffusion large language models},
  author={Wang, Yinjie and Yang, Ling and Li, Bowen and Tian, Ye and Shen, Ke and Wang, Mengdi},
  booktitle={International Conference on Learning Representations},
  volume={2026},
  pages={101476--101501},
  year={2026}
}

@article{dat2026dopsd,
  title={dOPSD: On-Policy Self-Distillation for Diffusion Language Models},
  author={Dat, Phuong Tuan and Li, Qi and Wang, Xinchao},
  journal={arXiv preprint arXiv:2607.04428},
  year={2026}
}

@article{luo2026learning,
  title={Learning from the Self-future: On-policy Self-distillation for dLLMs},
  author={Luo, Yifu and Chen, Zeyu and Wang, Haoyu and Hu, Xinhao and Zhang, Yuxuan and Sha, Zhizhou and Liu, Shiwei},
  journal={arXiv preprint arXiv:2606.18195},
  year={2026}
}

@article{su2026data,
  title={Data-Efficient Autoregressive-to-Diffusion Language Models via On-Policy Distillation},
  author={Su, Xingyu and Helwig, Jacob and Parashar, Shubham and Chagi, Atharv and Jotsna, Lakshmi and Zhi, Degui and Caverlee, James and Kalathil, Dileep and Ji, Shuiwang},
  journal={arXiv preprint arXiv:2606.06712},
  year={2026}
}

@article{zhao2026self,
  title={Self-distilled reasoner: On-policy self-distillation for large language models},
  author={Zhao, Siyan and Xie, Zhihui and Liu, Mengchen and Huang, Jing and Pang, Guan and Chen, Feiyu and Grover, Aditya},
  journal={arXiv preprint arXiv:2601.18734},
  year={2026}
}

@inproceedings{agarwal2024policy,
  title={On-policy distillation of language models: Learning from self-generated mistakes},
  author={Agarwal, Rishabh and Vieillard, Nino and Zhou, Yongchao and Stanczyk, Piotr and Ramos Garea, Sabela and Geist, Matthieu and Bachem, Olivier},
  booktitle={International Conference on Learning Representations},
  volume={2024},
  pages={21246--21263},
  year={2024}
}

@article{lu2025onpolicydistillation,
  author = {Kevin Lu and Thinking Machines Lab},
  title = {On-Policy Distillation},
  journal = {Thinking Machines Lab: Connectionism},
  year = {2025},
  note = {https://thinkingmachines.ai/blog/on-policy-distillation},
  doi = {10.64434/tml.20251026},
}

@article{zhang2026continuous,
  title={A Continuous-Time Reinforcement Learning Framework for Fine-Tuning Discrete Diffusion Models},
  author={Zhang, Zikun and Sheng, Jiayuan and Yao, David D and Tang, Wenpin},
  journal={arXiv preprint arXiv:2607.14522},
  year={2026}
}

@inproceedings{
zhao2025d,
title={d1: Scaling Reasoning in Diffusion Large Language Models via Reinforcement Learning},
author={Siyan Zhao and Devaansh Gupta and Qinqing Zheng and Aditya Grover},
booktitle={The Thirty-ninth Annual Conference on Neural Information Processing Systems},
year={2025},
url={https://openreview.net/forum?id=7ZVRlBFuEv}
}

@inproceedings{
yang2026mmada,
title={{MM}a{DA}: Multimodal Large Diffusion Language Models},
author={Ling Yang and Ye Tian and Bowen Li and Xinchen Zhang and Ke Shen and Yunhai Tong and Mengdi Wang},
booktitle={The Thirty-ninth Annual Conference on Neural Information Processing Systems},
year={2026},
url={https://openreview.net/forum?id=wczmXLuLGd}
}

@inproceedings{wang2026d2,
  title={d2: Improved techniques for training reasoning diffusion language models},
  author={Wang, Guanghan and Turok, Gilad and Schiff, Yair and Arriola, Marianne and Kuleshov, Volodymyr},
  booktitle={The Forty-third International Conference on Machine Learning},
  year={2026}
}

@inproceedings{
ou2026principled,
title={Principled {RL} for Diffusion {LLM}s Emerges from a Sequence-Level Perspective},
author={Jingyang Ou and Jiaqi Han and Minkai Xu and Shaoxuan Xu and Jianwen Xie and Stefano Ermon and Yi Wu and Chongxuan Li},
booktitle={The Fourteenth International Conference on Learning Representations},
year={2026},
url={https://openreview.net/forum?id=S5YeC9llIL}
}

@inproceedings{gu2024minillm,
  title={Minillm: Knowledge distillation of large language models},
  author={Gu, Yuxian and Dong, Li and Wei, Furu and Huang, Minlie},
  booktitle={International Conference on Learning Representations},
  volume={2024},
  pages={32694--32717},
  year={2024}
}

@article{yang2026learning,
  title={Learning beyond teacher: Generalized on-policy distillation with reward extrapolation},
  author={Yang, Wenkai and Liu, Weijie and Xie, Ruobing and Yang, Kai and Yang, Saiyong and Lin, Yankai},
  journal={arXiv preprint arXiv:2602.12125},
  year={2026}
}

@inproceedings{
kim2025train,
title={Train for the Worst, Plan for the Best: Understanding Token Ordering in Masked Diffusions},
author={Jaeyeon Kim and Kulin Shah and Vasilis Kontonis and Sham M. Kakade and Sitan Chen},
booktitle={Forty-second International Conference on Machine Learning},
year={2025},
url={https://openreview.net/forum?id=DjJmre5IkP}
}

@article{ben2026accelerated,
  title={Accelerated sampling from masked diffusion models via entropy bounded unmasking},
  author={Ben-Hamu, Heli and Gat, Itai and Severo, Daniel and Nolte, Niklas S and Karrer, Brian},
  journal={Advances in Neural Information Processing Systems},
  volume={38},
  pages={55981--56007},
  year={2026}
}

@inproceedings{chang2022maskgit,
  title={Maskgit: Masked generative image transformer},
  author={Chang, Huiwen and Zhang, Han and Jiang, Lu and Liu, Ce and Freeman, William T},
  booktitle={2022 IEEE/CVF Conference on Computer Vision and Pattern Recognition (CVPR)},
  pages={11305--11315},
  year={2022},
  organization={IEEE}
}

@inproceedings{lou2023discrete,
  title={Discrete diffusion language modeling by estimating the ratios of the data distribution},
  author={Lou, Aaron and Meng, Chenlin and Ermon, Stefano},
  booktitle={International Conference on Machine Learning},
  year={2024}
}

@article{campbell2022continuous,
  title={A continuous time framework for discrete denoising models},
  author={Campbell, Andrew and Benton, Joe and De Bortoli, Valentin and Rainforth, Thomas and Deligiannidis, George and Doucet, Arnaud},
  journal={Advances in Neural Information Processing Systems},
  volume={35},
  pages={28266--28279},
  year={2022}
}

@article{cobbe2021training,
  title={Training verifiers to solve math word problems},
  author={Cobbe, Karl and Kosaraju, Vineet and Bavarian, Mohammad and Chen, Mark and Jun, Heewoo and Kaiser, Lukasz and Plappert, Matthias and Tworek, Jerry and Hilton, Jacob and Nakano, Reiichiro},
  journal={arXiv preprint arXiv:2110.14168},
  year={2021}
}

@inproceedings{lightman2024let,
  title={Let's verify step by step},
  author={Lightman, Hunter and Kosaraju, Vineet and Burda, Yuri and Edwards, Harrison and Baker, Bowen and Lee, Teddy and Leike, Jan and Schulman, John and Sutskever, Ilya and Cobbe, Karl},
  booktitle={International Conference on Learning Representations},
  volume={2024},
  pages={39578--39601},
  year={2024}
}

@misc{aime24,
      title={American Invitational Mathematics Examination (AIME) 2024}, 
      author={Zhang, Yifan and Math-AI, Team},
      year={2024},
}

@article{song2026survey,
  title={A survey of on-policy distillation for large language models},
  author={Song, Mingyang and Zheng, Mao},
  journal={arXiv preprint arXiv:2604.00626},
  year={2026}
}

@article{huang2026hesitation,
  title={Hesitation-aware on-policy distillation for diffusion language models},
  author={Huang, Jianguo and Wan, Lipeng and Deng, Yanchen and An, Bo},
  journal={arXiv preprint arXiv:2609.33301},
  year={2026}
}

@article{wang2026remasking,
  title={Remasking discrete diffusion models with inference-time scaling},
  author={Wang, Guanghan and Schiff, Yair and Sahoo, Subham and Kuleshov, Volodymyr},
  journal={Advances in Neural Information Processing Systems},
  volume={38},
  pages={147282--147339},
  year={2026}
}

@article{sun2026step,
  title={Step-Level On-Policy Distillation: Interpolating Between On-Policy Distillation and Supervised Fine-Tuning},
  author={Sun, Changhui and Liu, Lanbo and Lei, Hang and Ling, Tong and Xie, Jiahang and Zheng, Zhiyong and Wang, Yujia and Liu, Hao and Xiao, Feng and Liu, Lu et al},
  journal={arXiv preprint arXiv:2608.16333},
  year={2026}
}

@article{schulman2017proximal,
  title={Proximal policy optimization algorithms},
  author={Schulman, John and Wolski, Filip and Dhariwal, Prafulla and Radford, Alec and Klimov, Oleg},
  journal={arXiv preprint arXiv:1707.06347},
  year={2017}
}

@misc{patiño2025_unlocking_on_policy_distillation_for_any_model_family,
  title={Unlocking On-Policy Distillation for Any Model Family},
  author={Carlos Miguel Patiño and Kashif Rasul and Quentin Gallouédec and Ben Burtenshaw and Sergio Paniego and Vaibhav Srivastav and Thibaud Frere and Ed Beeching and Lewis Tunstall and Leandro von Werra and Thomas Wolf},
  year={2025},
  howpublished={\url{https://huggingface.co/spaces/HuggingFaceH4/on-policy-distillation}},
}

@article{li2026rethinking,
  title={Rethinking On-Policy Distillation of Large Language Models: Phenomenology, Mechanism, and Recipe},
  author={Li, Yaxuan and Zuo, Yuxin and He, Bingxiang and Zhang, Jinqian and Xiao, Chaojun and Qian, Cheng and Yu, Tianyu and Gao, Huan-ang and Yang, Wenkai and Liu, Zhiyuan and Ding, Ning},
  journal={arXiv preprint arXiv:2604.13016},
  year={2026}
}

@article{ding2026does,
  title={Does On-Policy Distillation Really Distill? From Noisy Teacher to Self-Improvement},
  author={Ding, Yi and Zhang, Ruqi},
  journal={arXiv preprint arXiv:2608.31046},
  year={2026}
}

@article{yu2026understanding,
  title={Understanding On-Policy Distillation: A Mechanistic Interpretability Perspective via Sparse Crosscoders},
  author={Yu, Zichao and Ye, Qianshuo and Wang, Xu and Zou, Difan},
  journal={arXiv preprint arXiv:2609.35210},
  year={2026}
}

@article{fu2026rethinking,
  title={Rethinking On-Policy Distillation of Large Language Models {II}: One Training Example},
  author={Fu, Zixuan and He, Bingxiang and Zuo, Yuxin and Huang, Haohuan and Zhang, Jinqian and Xiao, Ruhang and Qian, Cheng and Luo, Qinyu and Gao, Huan-ang and Wang, Yudong and Liu, Zhiyuan and Ding, Ning and Xiao, Chaojun},
  journal={arXiv preprint arXiv:2609.04172},
  year={2026}
}

@article{liu2026forkleft,
  title={ForkLeft: Entropy-First Rollouts for Prefix-Aligned Autoregressive-to-Diffusion Distillation},
  author={Liu, Junming and Wang, Jicheng and He, Yifeng and Chen, Hao and Qi, Jianzhong},
  journal={arXiv preprint arXiv:2609.32448},
  year={2026}
}

@inproceedings{
fu2026efficientdlm,
title={Efficient-{DLM}: From Autoregressive to Diffusion Language Models, and Beyond in Speed},
author={Yonggan Fu and Lexington Whalen and Zhifan Ye and Xin Dong and Shizhe Diao and Jingyu Liu and Chengyue Wu and Hao Zhang and Enze Xie and Song Han and Maksim Khadkevich and Jan Kautz and Yingyan Celine Lin and Pavlo Molchanov},
booktitle={Forty-third International Conference on Machine Learning},
year={2026},
url={https://openreview.net/forum?id=ynRnOqI2W8}
}
